\documentclass[11pt]{article}
\usepackage[margin=1.25in]{geometry}
\RequirePackage[OT1]{fontenc}
\RequirePackage{amsthm,amsmath,amsmath,amsfonts,amssymb,mathabx}
\RequirePackage[colorlinks,citecolor=blue,urlcolor=blue]{hyperref}
\usepackage{times}
\usepackage{multirow}
\usepackage{enumitem}

\usepackage{graphicx}

\usepackage{titling}
\usepackage{authblk}

\newtheorem{theorem}{Theorem}
\newtheorem{lemma}{Lemma}
\newtheorem{corollary}{Corollary}

\newtheorem{proposition}{Proposition}

\theoremstyle{definition}

\let\oldnl\nl% Store \nl in \oldnl
\newcommand{\nonl}{\renewcommand{\nl}{\let\nl\oldnl}}

\usepackage[linesnumbered,boxed]{algorithm2e}
\SetKwRepeat{Do}{do}{while}%
\SetKwInput{KwInput}{Input}
\SetKwInput{KwOutput}{Output}
\SetKwInput{KwFunction}{Function}
\SetKwInput{KwParameters}{Parameters}
\SetKwComment{Comment}{$\triangleright$\ }{}
\SetCommentSty{itshape}

\newcommand{\kl}{\mathrm{KL}}

\def\reg{\mathrm{Reg}}
\def\tv{\mathrm{TV}}

\renewcommand{\hat}{\widehat}
\renewcommand{\tilde}{\widetilde}

\newcommand{\citep}{\cite}
\newcommand{\citet}{\cite}

\begin{document}

\title{An Optimal Policy for Dynamic Assortment Planning Under Uncapacitated Multinomial Logit Models}
\author[1]{Xi Chen \footnote{Author names listed in alphabetical order.}}
\author[2]{Yining Wang}
\author[3]{Yuan Zhou}
\affil[1]{Leonard N. Stern School of Business, New York University}
\affil[2]{Machine Learning Department, Carnegie Mellon University}
\affil[3]{Computer Science Department, Indiana University at Bloomington}

\maketitle

\begin{abstract}
We study the dynamic assortment planning problem, where for each arriving customer, the seller offers an assortment of substitutable products and customer makes the purchase among offered products according to an uncapacitated multinomial logit (MNL) model.  Since all the utility parameters of MNL are unknown, the seller needs to simultaneously learn customers' choice behavior and make dynamic decisions on assortments based on the current knowledge. The goal of the seller is to maximize the expected revenue, or equivalently, to minimize the expected regret. Although dynamic assortment planning problem has received an increasing attention in revenue management, most existing policies require the estimation of mean utility for each product and the final regret usually involves the number of products $N$.  The optimal regret of the dynamic assortment planning problem under the most basic and popular choice model---MNL model is still open. By carefully analyzing a revenue potential function, we develop a trisection based policy combined with adaptive confidence bound construction, which achieves an \emph{item-independent} regret bound of $O(\sqrt{T})$, where $T$ is the length of selling horizon. We further establish the matching lower bound result to show the optimality of our policy.
There are two major advantages of the proposed policy. First, the regret of all our policies has no dependence on $N$. Second, our policies are almost assumption free: there is no assumption on mean utility nor any ``separability'' condition on the expected revenues for different assortments.	
Our result also extends the unimodal bandit literature.

\textbf{Keywords}: dynamic assortment optimization, multinomial logit choice model, trisection algorithm, regret analysis.
\end{abstract}

\section{Introduction}

Assortment planning has a wide range of applications in retailing and online advertising. Given a large number of substitutable products,
the assortment planning problem refers to the selection of a subset of products (a.k.a., an assortment) offering to a customer such that the expected revenue is maximized. To model customers' choice behavior when facing a set of offered products, discrete choice models have been widely used, which  capture demand for each product as a function of the entire assortment. One of the most popular discrete choice models is the \emph{the multinomial logit model (MNL)}, which is naturally resulted from the random utility theory where a customer's preference of a product is represented by the mean utility of the product with a random factor \citep{McFadden1974}. In many scenarios, customers' choice behavior (e.g., mean utilities of products) may not be given as \emph{a priori} and cannot be easily estimated well due to the insufficiency of historical data (e.g.,  fast fashion sale or online advertising). To address this challenge,  dynamic assortment planning that simultaneously learns choice behavior and makes decisions on the assortment has received a lot of attentions \citep{Caro2007, Rusmevichientong2010, Saure2013, Agrawal16MNLBandit, Agrawal17Thompson}. More specifically, in a dynamic assortment planning problem, the seller offers an assortment to each arriving customer in a finite time horizon of length $T$. The goal of the seller is to maximize the cumulative expected revenue over $T$ periods, or equivalently, to minimize the \emph{regret},  which is defined as the gap between the  expected revenue generated by the policy and the oracle expected revenue when the mean utility for each product is known as \emph{a priori}.  %{\red (Yining: should we highlight worst-case regret?)}

{Despite a lot of research in the area of dynamic assortment planning under various choice models (see Section \ref{sec:related}), the optimal policy for the most fundamental uncapacitated MNL model still remains open in the literature.}  A natural idea to tackle this problem is to conduct some form of maximum likelihood estimation (MLE) of mean utilities of different products on-the-fly, and then select the assortment that maximizes the expected revenue based on the current estimate of mean utilities. However,   when the number of products $N$ is large as compared to the horizon length $T$, accurate estimation of mean utilities is extremely difficult, if not impossible, without additional assumptions. In terms of regret analysis, this approach usually incurs a regret that is polynomial in $N$, which is sub-optimal according to our lower bound result (i.e., $\Omega(\sqrt{T})$). 
Therefore, the following question naturally arises: can we design dynamic assortment policies without explicit estimation of mean utilities and achieve the optimal regret that is independent of $N$?

In this paper, we provide affirmative answers to this question under the most fundamental and popular uncapacitated multinominal logit model. As mentioned above, the estimation of utility parameters will be inaccurate when $N$ is large and thus existing methods based on maximum likelihood estimation cannot be directly used. We design several new techniques to address  the challenge. Under an MNL model, we leverage the structure of the optimal assortment in static problems and convert the problem into a \emph{dynamic optimization} of a carefully designed \emph{potential function}. In particular, the seminal result by \cite{Talluri2004,Gallego2004,Liu2008} shows that the optimal assortment belongs to the set of revenue-ordered assortments. More precisely, assuming that $N$ products are revenue-ordered with the revenues  $r_1 \geq r_2 \geq \ldots \geq r_N$, then the optimal assortment must belong to the set $\{\{\}, \{1\}, \{1,2\}, \ldots, \{1,\ldots, N\}\}$. Therefore, it suffices to only consider the following level sets of products: for each cutoff parameter $\theta\geq 0$,  we define the \emph{level set} to be the products whose revenue is greater than or equal to $\theta$. Further, motivated by \cite{Rusmevichientong2010}, we can define the \emph{potential function} $F(\theta)$ to be the expected revenue when this level set is offered as an assortment.  

To construct our policy, we first establish a set of important properties of the potential function $F(\theta)$, including 1) we show that the fixed point of $F(\theta)$ is the maximizer $\theta^*$ and leads to the optimal assortment; 2) we set up a reference line and comparing $F(\theta)$ with the reference line {to decide whether $F$ is increasing or decreasing locally at $\theta$}. Based on these properties, we propose a trisection search policy that dynamically searches the maximizer $\theta^*$ of the potential function and achieves an optimal regret up to logarithmic factors in $T$. {Then we further develop an approach with adaptive confidence levels to remove the logarithmic factor in $T$.}  The matching lower bound result has also been established, which shows the optimality of the proposed policy. By exploring the structure of the potential function, we no longer need to estimate $N$ parameters of mean utilities; instead, we only estimate the expected revenue of level sets at a few cutoff points. Before we present an overview of our technical result in Sec.~\ref{sec:result}, we briefly highlight two important advantages of the proposed policies. 
\begin{enumerate}
	\item First, the regrets of our policies have no dependence on the number of products $N$.  This property makes our result more favorable for scenarios when a large number of potential items are available, e.g., online sales or online advertisement. And a key message behind this result is that by exploring the structure of the problem, the explicit estimation of utility parameters could be avoided in dynamic assortment planning.
	
	\item Second, our policy is almost assumption-free: we only require the revenue for each product is upper bounded by a constant and the knowledge of total selling horizon $T$, which is usually available in practice. We have no assumption on the mean utilities (e.g., the assumption that the no-purchase is the most frequent outcome as in \cite{Agrawal16MNLBandit,Agrawal17Thompson}). 	This relaxation of assumptions is possible because we do \emph{not} attempt to estimate individual mean utilities in our algorithms. Moreover,  we do not have any ``separation condition'' on the expected revenue between a pair of candidate assortments, which has been assumed in the existing literature \citep{Rusmevichientong2010, Saure2013}.
	
\end{enumerate}

\subsection{Our results and techniques}
\label{sec:result}
The main contribution of this paper is an optimal characterization of the worst-case regret for dynamic assortment planning under the MNL model.
More specifically, we have the following informal statement of the main results in this paper.

\begin{theorem}[informal]
There exists a policy whose worst-case regret over $T$ time periods is upper bounded by $C_1\sqrt{T}$ %asymptotically almost surely
%\footnote{A sequence of events $\{E_T\}_{T\geq 1}$ holds asymptotically almost surely if $\lim_{T\to\infty}\Pr[E_T]=1$.}
 for some universal constant $C_1>0$;
furthermore, there exists another universal constant $C_2>0$ such that no policy can achieve a  worst-case regret smaller than $C_2\sqrt{T}$.
\label{thm:informal}
\end{theorem}

%An important aspect of Theorem \ref{thm:informal} is that our regret bound is completely \emph{independent of the number of items $N$}.

%{\color{blue}and also makes the result applicable to scenarios when a large number of potential items are available, such as Amazon online sales systems.}

%albeit with quite different analysis.
%an iterated logarithmic term in $T$;
%whether it can be removed to obtain bounds matching up to numerical constants remains an open question, as we discuss in Sec.~\ref{sec:discussion}.
%
To enable such an $N$-independent regret, we provide a refined analysis of a certain \emph{unimodal} revenue potential function first studied in \cite{rusmevichientong2012robust}
and consider a trisection algorithm on revenue levels, extending some  ideas in unimodal bandits on either discrete or continuous arm domains \cite{yu2011unimodal,combes2014unimodal,agarwal2013stochastic}.
An important challenge in our problem is that the revenue potential function (defined in Eq.~(\ref{eq:F})) does not satisfy convexity or local Lipschitz growth,
and therefore previous results on unimodal bandits cannot be directly applied (see the related work section \ref{sec:related} for details).
Moreover, it is a simple exercise that mere unimodality in multi-armed bandits cannot lead to regret smaller than $\sqrt{NT}$,
because the worst-case constructions in the classical lower bound in multi-armed bandits are based on unimodal arms (see, e.g., \cite{Bubeck:Survey:12,bubeck2009pure}).

To overcome these difficulties, we establish additional properties of the revenue potential function which are different from classical convexity or Lipschitz growth properties.
In particular, we prove connections between the potential function and the straight line $F(\theta)=\theta$, which is then used as guidelines
in our update rules of trisection.
Also, because the potential function behaves differently on $F(\theta)\leq \theta$ and $F(\theta)\geq \theta$,
our trisection algorithm is \emph{asymmetric} in the treatments of the two trisection mid-points,
which is in contrast to previous trisection based methods for unimodal bandits \cite{yu2011unimodal,combes2014unimodal}
that treat both trisection mid-points symmetrically.

We also remark that the trisection search policy leads to a regret $O(\sqrt{T\log T})$, where the optimal regret should be $\Theta(\sqrt{T})$. The removal of additional $\log T$ terms in dynamic assortment selection and unimodal bandit problems is quite non-trivial, which requires new technical development.  In fact, most previous results on dynamic assortment selection \cite{rusmevichientong2012robust,Agrawal16MNLBandit,Agrawal17Thompson} and unimodal/convex bandits \cite{yu2011unimodal,combes2014unimodal,agarwal2013stochastic} have additional $\log T$ terms in regret upper bounds. 
The removal of this $\log(T)$ term is achieved by using confidence bounds with adaptively chosen confidence levels corresponding to different amounts of data collected.
At a higher level, our  strategy shares a similar spirit to the MOSS (Minimax Optimal Strategy in the Stochastic case) algorithm for multi-armed bandits \cite{audibert2009minimax}. On the other hand, the analysis is quite different from the analysis of the MOSS algorithm, involving new concentration inequalities and induction arguments tailored specifically to our model and proposed policy.

The rest of the paper is organized as follows. Sec.~\ref{sec:related} discusses the related work from both revenue management and bandit learning fields. We introduce the model and notations in Sec.~\ref{sec:model}. We further define the revenue potential function and investigate its properties in Sec.~\ref{sec:potential}. The policy and regret analysis will be provided in Sec.~\ref{sec:policy} and the lower bound results are developed in Sec.~\ref{sec:lower}. In Sec.~\ref{sec:simulation}, we provide some simulation studies to illustrate the performance of the proposed policies and conclusion and discussions will be followed in Sec.~\ref{sec:discussion}. Some technical proofs will be relegated to the supplementary material.

\section{Related work}\label{sec:related}

There are two lines of related work --- dynamic assortment planning and unimodal bandits. We will provide a brief review of both fields and highlight some closely related work.

\subsection{Dynamic assortment planning}

Static assortment planning with known choice behavior has been an active research area since the seminal work by \citet{Ryzin1999,Mahajan2001}. When the customer makes the choice according to the MNL model, \citet{Talluri2004,Gallego2004} prove the the optimal  assortment will belong to revenue-ordered assortments (see Lemma \ref{lem:popular} in Sec.~\ref{sec:potential}). An alternative proof is provided in \cite{Liu2008}. This important structural result enables efficient computation of static assortment planning under the MNL model, which reduces the number of candidate assortments from $2^N$ to $N$ and will also be used in our policy development.

Motivated by the large-scale online retailing, researchers start to relax the assumption on  prior knowledge of customers' choice behavior. The question of \emph{dynamic} optimization of assortments has received increasing attention in both the machine learning and operations management society \cite{Caro2007,Rusmevichientong2010,Saure2013,Agrawal16MNLBandit,Agrawal17Thompson},
where the mean utilities of products are unknown and have to be learnt on the fly.  Motivated by fast-fashion retailing, the work by \cite{Caro2007} was the first to study dynamic assortment planning problem, which assumes that the demand for product is independent of each other. The work \cite{Rusmevichientong2010} and \cite{Saure2013} incorporate choice models of MNL into dynamic assortment planning and formulate the problem into a online regret minimization problem. 

The work  \cite{Rusmevichientong2010} is closely related to our paper, which analyzes the same revenue potential function and proposes a golden ratio search algorithm based on the unimodal property of the potential function. However, only using the unimodal property leads a regret bound involving $\log(N)$ \cite{Rusmevichientong2010} , which is not $N$-independent. Moreover,  the golden ratio search algorithm imposes a strong ``separability assumption'' (see Proposition 8 in \cite{Rusmevichientong2010}), which assumes a constant gap between the expected revenues of any pair of candidate assortments,  which may fail when the number of items $N$ is large. 
In this work we relax the gap assumption and also remove the additional $\log N$ dependency by a more refined analysis of properties of the revenue potential function.
%and borrowing ``trisection'' ideas from the unimodal bandit literature \cite{yu2011unimodal,combes2014unimodal,agarwal2013stochastic}.

Our paper is also closely related to recent work  \cite{Agrawal16MNLBandit} and \cite{Agrawal17Thompson}. These work develop variants of UCB and Thompson sampling type methods for \emph{capacitated} MNL assortment models, where the size of each assortment is not allowed to exceed a pre-specified parameter $K$. Here the capacity limit $K$ is usually much smaller than $N$.
For the capacitated MNL model, the paper \cite{chen2018note} further establishes a lower bound result, which shows an $\Omega(\sqrt{NT})$ regret lower bound exists provided that $K \leq N/4$. By comparing this result with our result described in Theorem \ref{thm:informal}, it is interesting to see that the regret behavior in capacitated and uncapacitated MNL models is significantly different (see Table \ref{tab:regret}). While the dependence on $N$ in regret is unavoidable in the capacitated case, this paper shows that it can be got rid of in the uncapacitated case. We remove this dependence on $N$ by designing a novel policy that does not explicitly estimate utility parameters.  %We also note that the uncapacitated MNL can be viewed as a special capacitated MNL with the capacity limit $K=N$. 
%As a future work, it would be interesting to  investigate the phase transition from  the regret $\Theta(\sqrt{NT})$ (when $K\leq N/4$) to $\Theta(\sqrt{T})$ (when $K=N$).} \xnote{Should we mention it here?}
%{\color{red}\bf [Yining: I don't think mentioning $N/4<K<N$ is a good idea here because we don't know how to do it and don't know the regret in such a scenario.]}

In addition to MNL models, there are some recent work studying dynamic assortment under more complicated choice models, such as nested logit models \cite{Chen:18dynamic} and contextual MNL models \cite{Wang:17:person, Chen:18context}.  We also note that to highlight our key idea and focus on the balance between information collection and revenue maximization, we study stylized dynamic assortment planning problems following the existing literature \citep{Rusmevichientong2010, Saure2013, Agrawal16MNLBandit, Agrawal17Thompson}, which ignore operational considerations such as price decisions and inventory replenishment.

\begin{table}
	\center
	\caption{Summary of the state-of-the-art worst-case regrets for dynamic assortment planning under uncapacitated MNL and capacitated MNL, where $T$ and $N$ denote the length of the horizon and the number of products, respectively. We also provide the reference for each result, either the theorem number (when the result is first derived in this paper) or the reference.  Here, the tilde-$O$ notation $\tilde{O}$ is used as a variant of the standard big-$O$ notation but hides logarithmic factors. }
	\begin{tabular}{|ccc|} \hline
		Worst-case Regret & uncapacitated MNL  & capacitated MNL ($K\leq N/4$)     \\ \hline
		\multirow{2}{*}{Upper bound} & $O(\sqrt{T})$  & $\tilde{O}(\sqrt{NT}+N)$   \\  
   		& (Theorem \ref{thm:upper-lil}) & \citep{Agrawal16MNLBandit,Agrawal17Thompson}    \\ \hline
		\multirow{2}{*}{Lower bound} & $\Omega(\sqrt{T})$    & $\Omega(\sqrt{NT})$               \\ 
		& (Theorem \ref{thm:lower}) & \citep{chen2018note}    \\ \hline
	\end{tabular}
	\label{tab:regret}
\end{table}

\subsection{Unimodal bandits}

Another relevant line of research is \emph{unimodal bandit} \cite{yu2011unimodal,combes2014unimodal,agarwal2013stochastic,cope2009regret},
in which discrete or continuous multi-armed bandit problems are considered with additional unimodality constraints on the means of the arms.
Apart from unimodality, additional structures such as ``inverse Lipschitz continuity'' (e.g., $|\mu(i)-\mu(j)|\geq L|i-j|$ for some constant $L$, where $\mu(i)$ denotes the mean reward of the $i$-th arm)
or convexity are imposed to ensure {smaller regret compared to unstructured multi-armed bandits}.
%\xnote{What do we mean by ``improvement of the regret''}. 
However, both conditions fail to hold for the revenue potential function arising from uncapacitated MNL-based assortment planning problems.
%We also note that, in one previous work \cite{combes2014unimodal}, ``gap'' assumptions were also made to derive ``gap-dependent'' regret bounds, which are not directly comparable to our regret bounds in a ``gap-free'' setting. \xnote{do we need to mention gap assumption?}
In addition, under the gap-free setting where an $O(\sqrt{T})$ regret is to be expected, most previous works have additional $\log T$ terms in their regret upper bounds
(except for the work of \cite{cope2009regret} which introduces additional strong regularity conditions on the underlying functions).
%\xnote{Do we really want to highlight this $\log T$ term?}
In \cite{cohen2016online}, a more general problem of optimizing piecewise-constant function is considered, without assuming a unimodal structure
of the function. Consequently, a weaker $\widetilde O(T^{2/3})$ regret is derived.

%\xnote{Some additional papers, please add into an appropriate place: The work of \cite{combes2014unimodal} also derived gap-dependent regret bounds for unimodal bandit,
%	which is not easily comparable to our bounds.}
\section{Model specification}
\label{sec:model}

Let $\mathcal N$ be a finite set of all products/items with $|\mathcal N|=N$, and each item $i\in\mathcal N$ is associated with a revenue parameter $r_i>0$ and a  utility parameter (a.k.a., preference parameter) $v_i\geq 0$.\footnote{From random utility theory, we have $v_i =\exp(u_i)$, where $u_i$ is the underlying mean utility. For the ease of presentation, we will call $v_i$ the ``utility parameter'' since we only use $v_i$ throughout this paper.}
Throughout the paper we conveniently label all items in $\mathcal N$ as $1,2,\cdots,N$.
The revenue parameters $r_1,\cdots,r_N$ are known to the retailer, who has full knowledge of each items' price/cost; while the utility parameters $v_1,\cdots,v_N$ are unknown.
Let $\mathbb S=2^{\mathcal N}$ be the set of all possible assortments. At every time time $t$, a retailer picks an assortment $S_t\in \mathbb S$ ($S_t\neq\emptyset$), and observes a purchasing action $i_t\in S_t\cup \{0\}$,
where $i_t=0$ means no purchase occurs at time $t$. 
If a purchasing action is made (i.e., $i_t\neq 0$), the corresponding revenue $r_{i_t}$ is collected. It is worthy noting that since items are substitutable, a typical setting of assortment planning usually restricts each purchase to be a single item.

The distribution of $i_t$ is modeled by the following multinomial-logit (MNL) model:
\begin{equation}
\Pr[i_t=j] = \left\{\begin{array}{ll}{v_j}/({1+\sum_{i\in S_t}v_i})& j\in S_t;\\ {1}/({1+\sum_{i\in S_t}v_i})& j=0.\end{array}\right.
\label{eq:mnl-bandit}
\end{equation}
Define also $R(S_t)$ as the \emph{expected} revenue by supplementing $S_t$ to a customer; more specifically,
\begin{equation}
R(S_t) := \sum_{j\in S_t}\Pr[i_t=j]\cdot r_j = \frac{\sum_{j\in S_t}r_jv_j}{1+\sum_{j\in S_t}v_j}.
\end{equation}
For normalization purposes the utility parameter for the ``no-purchase'' action is assumed to be $v_0=1$.
Apart from that, the rest of the preference parameters $\{v_i\}_{i=1}^N$ are \emph{unknown} to the retailer and have to be either explicitly or implicitly learnt
from customers' purchasing actions $\{i_t\}_{t=1}^T$.

The retailer's objective is to maximize the expected revenue over the $T$ time periods. Such an objective is equivalent to the ``regret minimization'', in which
the retailer's assortment sequence is compared against the optimal assortment. 
More specifically, the goal of the retailer is to design a policy $\pi$ that generates $\{S_t\}_{t=1}^T$ to minimize the following cumulative regret:
\begin{equation}
\reg(\{S_t\}_{t=1}^T)  := \sum_{t=1}^T{R(S^*) - \mathbb E^\pi[R(S_t)]} \;\;\;\;\;\text{where}\;\;S^* \in \arg\max_{S\in\mathbb S} R(S).
\end{equation}
Here,  $R(S_t) = \mathbb E[r_{i_t}|S_t]$ is the expected revenue the retailer collects on assortment $S_t$. For notational convenience we define $r_0=0$ corresponding to the ``no-purchase'' action.

Finally, throughout this paper we only make the following standard assumption on the revenue parameters (see, e.g., Theorem 1 in \cite{Agrawal16MNLBandit}):
\begin{enumerate}[leftmargin=0.5in]
	\item[(A1)] $r_\infty := \max_{i\in\mathcal N}r_i \leq 1$.
\end{enumerate}
We note that upper bound on the maximum revenue is assumed to be one without loss of generality, since one can always normalize the revenues.

\section{The revenue potential function and its properties}\label{sec:potential}

The set $\mathbb S$ consists of $2^N$ different assortments, which poses a significant challenge on both  regret minimization (treating each assortment in $\mathbb S$ independently results in exponentially large regret) and computation (as it is intractable to enumerate all assortments in $\mathbb S$). %computationally (it is intractable to enumerate all assortments in $\mathbb S$) and information-theoretically (treating each assortment in $\mathbb S$ independently results in exponentially large regret).
{ To address the challenge, we can reduce the number of candidate assortments in $\mathbb S$ by constraining
	such assortment selections to ``level sets''.}
In particular, for a given real number $\theta\geq 0$, define the $\theta$-\emph{level set} to be
\[
\mathcal L_\theta(\mathcal N) := \{i\in\mathcal N: r_i\geq \theta\}
\]
as all items whose revenues are not smaller than $\theta$. For notational simplicity, we will use $\mathcal L_\theta$ (omitting $\mathcal N$ in the parenthesis) when the context is clear.
Further, let
\begin{equation}
\mathbb P := \{\mathcal L_\theta(\mathcal N): \theta\geq 0\}\subseteq\mathbb S
\end{equation}
be the class of all candidate assortments in $\mathbb S$ that can be expressed as level sets. It is easy to verify that $|\mathbb P|\leq N$, which is significantly smaller than $|\mathbb S|=2^N$. 

{ It is well-known that the optimal expected revenue for the static assortment optimization problem will remain the same when reducing the candidate assortments from $\mathbb S$ to $\mathbb P$.}
%The following definition characterizes a much restrictive class of assortment selections, by constraining the selected assortments
%to contain only items with the highest ranked revenue parameters.
%\begin{definition}
%For any $\theta\geq 0$ define $\mathcal L_\theta(\mathcal N) := \{i\in\mathcal N: r_i\geq \theta\}$.
%Denote also $\mathbb P := \{\mathcal L_\theta(\mathcal N): \theta\geq 0\}\subseteq \mathbb S$.
%\label{defn:level-set}
%\end{definition}
More precisely, the following lemma is a classical result in  revenue management \citep{Talluri2004,Gallego2004,Liu2008},
which shows the optimal expected revenue can be achieved by only considering the restricted level set class $\mathbb P$ under the MNL model.
\begin{lemma}[\citep{Talluri2004,Gallego2004,Liu2008}] Under the MNL model, there exists an subset $S^*\subseteq \mathcal{N}$  such taht
	$
	R(S^*)=\max_{S\in\mathbb S}R(S) = \max_{S\in\mathbb P}R(S).
	$
	\label{lem:popular}
\end{lemma}
In other words, Lemma \ref{lem:popular}  suggests that it suffices to consider ``level-set'' type assortments $\mathcal L_\theta$
and to find $\theta\in[0,1]$ that gives rises to the largest $R(\mathcal L_\theta)$.

This motivates the following ``potential'' function, which takes a revenue threshold $\theta$ as input and outputs the expected revenue of its corresponding level set assortments:
\begin{equation}
\text{\it The revenue potential function:} \;\;\;\;F(\theta) := R(\mathcal L_\theta), \;\;\theta\in[0,1].
\label{eq:F}
\end{equation}

Intuitively, $F(\theta)$ is the expected revenue obtained by providing the assortment consisting of all items whose revenues exceed or are equal to $\theta$. The potential function  plays a central role in the development of our dynamic trisection search algorithm and item-independent regret bounds.  Similar idea of studying the expected revenue of revenue-ordered items was also considered in \cite{rusmevichientong2012robust}. But we will derive a more comprehensive list of properties of the potential function $F$ to facilitate our algorithmic development and analysis. The derived properties in this section could also be potentially useful for solving other assortment planning problems under the MNL.

Because item revenues $r_i$ are discrete, $F$ is a piecewise-constant function as illustrated in the left picture in Fig.~\ref{fig:F}, where $\mathcal S=\{s_1,\cdots,s_m\}$ are the changing points of $F$.
More specifically, we have the following proposition and its verification is easy from the definition and the discretized nature of $F$.
\begin{proposition}
	There exists $c_0,\cdots,c_m\geq 0$ satisfying $c_i\neq c_{i+1}$ for all $i=0,\cdots,m-1$, and $\mathcal S=\{s_1,\cdots,s_m\}\subseteq \{r_i\}_{i=1}^N$, such that
	\begin{equation}
	F(\theta) = c_0\cdot\mathbb I[\theta\leq s_1] + \sum_{i=1}^{m-1}{c_i\cdot\mathbb I[s_i < \theta\leq s_{i+1}]} + c_m\cdot \mathbb I[\theta>s_m],
	\end{equation}
	where $c_m=0$.
\end{proposition}

Define $F^* := \max_{0\leq i\leq m}c_i = \sup_{\theta\geq 0}F(\theta)$ as the maximum value of $F$.
By Lemma \ref{lem:popular}, we have the following corollary saying that $F^*$ equals the expected revenue of the optimal assortment.
\begin{corollary}
	$F^* = R(S^*)$.
	\label{cor:popular}
\end{corollary}

We further establish some more refined structural properties of $F$.
For notational simplicity, let $F(x^+) := \lim_{y\to x^+}F(y)$ and $F(x^-) := \lim_{y\to x^-}F(y)$.
\begin{lemma}
	There exists $\theta^*>0$ such that $\theta^*=F(\theta^*)=F^*$.
	\label{lem:rrequal}
\end{lemma}

\begin{lemma}
	For any $\theta\geq \theta^*$, $F(\theta)\leq \theta$ and $F(\theta)\geq F(\theta^+)$.
	\label{lem:Fmonotonic-right}
\end{lemma}

\begin{lemma}
	For any $\theta\leq \theta^*$, $F(\theta)\geq \theta$ and $F(\theta)\leq F(\theta^+)$.
	\label{lem:Fmonotonic-left}
\end{lemma}

The proofs of the above lemmas are given in the supplementary material. Lemmas \ref{lem:rrequal}, \ref{lem:Fmonotonic-right} and \ref{lem:Fmonotonic-left} provide a complete picture of the structure of the potential function $F$, and most importantly the relationship between $F$ and the central
straight line $F(\theta)=\theta$, as depicted in the right picture of Fig.~\ref{fig:F}.
In particular, $F$ intersects with the $y=x$ line at $\theta^*$ that attains the maximum function value $F^*$, and monotonically decreases as one moves away from $\theta^*$,
meaning that $F$ is \emph{uni-modal}.
Furthermore, Lemmas \ref{lem:Fmonotonic-right} and \ref{lem:Fmonotonic-left} show that { (1) $F$ is left-continuous; (2) $F^*$ lies below the $y=x$ line to the right of $\theta^*$
	and above the $y=x$ line to the left of $\theta^*$.}
This helps us judge the positioning of a particular revenue level $\theta$ by simply comparing the expected revenue of $R(\mathcal L_\theta)$ with $\theta$ itself,
motivating an asymmetric trisection algorithm which we describe in the next section.

%Therefore, the relative position of $\theta$ with respect to $\theta^*$ can be decided by comparing the value of $\theta$ to the value of  $F(\theta)$. This fact lays the foundation of our dynamic trisection search algorithm, which is detailed in the next section.

%The potential $F$ was first introduced and considered in \cite{kok2008assortment}, in which it was proved that $F$ is left-continuous, piecewise-constant and \emph{unimodal} in its input revenue $\theta$. Using such unimodality, a golden-ratio search based policy was designed that achieves $O(\log N\log T)$ regret under additional consecutive gap assumptions of the level set assortments $\{\mathcal L_\theta\}$. To derive gap-independent results and to get rid of the additional $\log N$ dependency, we provide a more refined analysis of properties of the potential function $F$ in this paper, summarized in the following three lemmas:

\begin{figure}[t]
\centering
\includegraphics[width=0.495\textwidth]{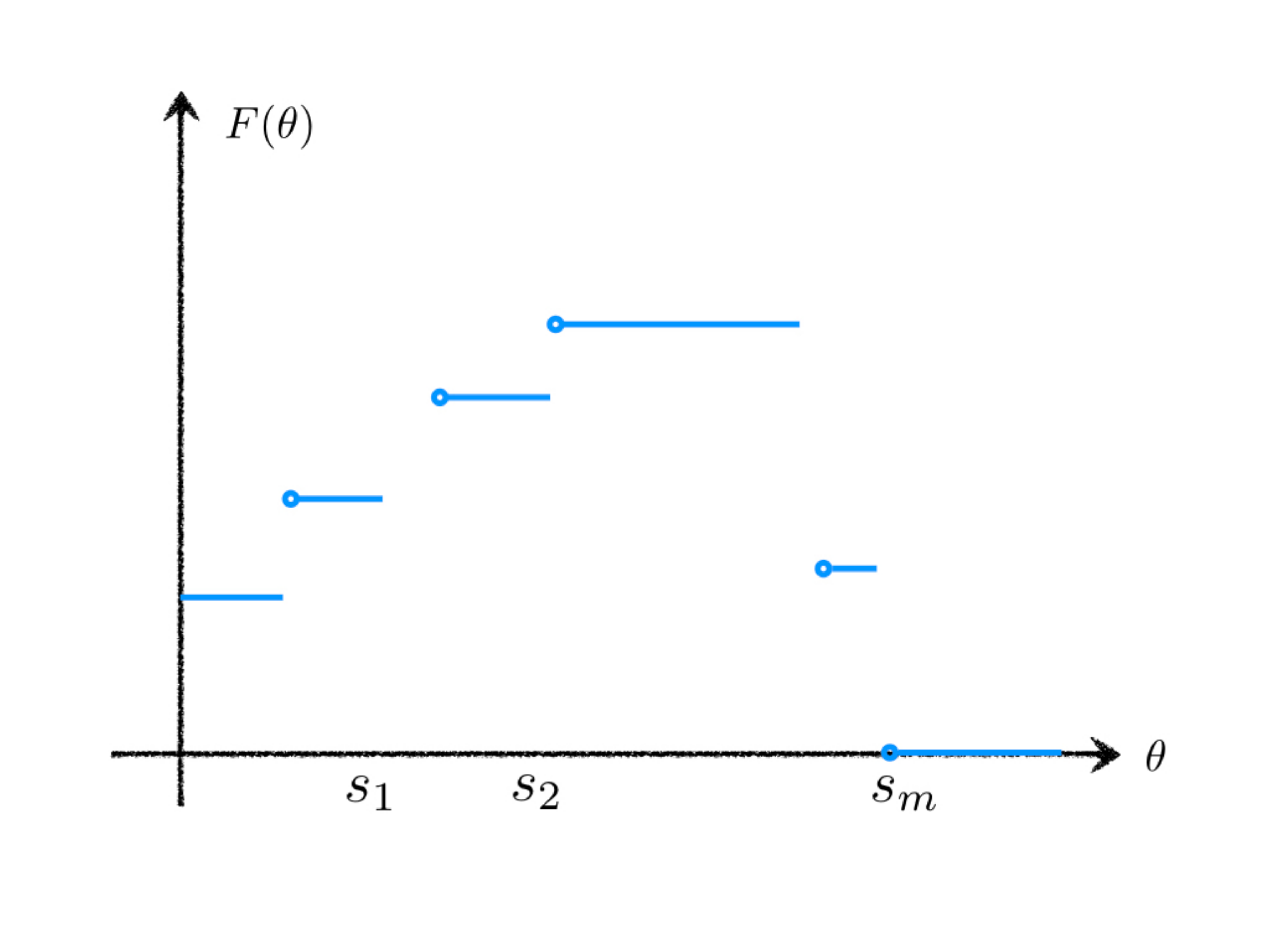}
\includegraphics[width=0.495\textwidth]{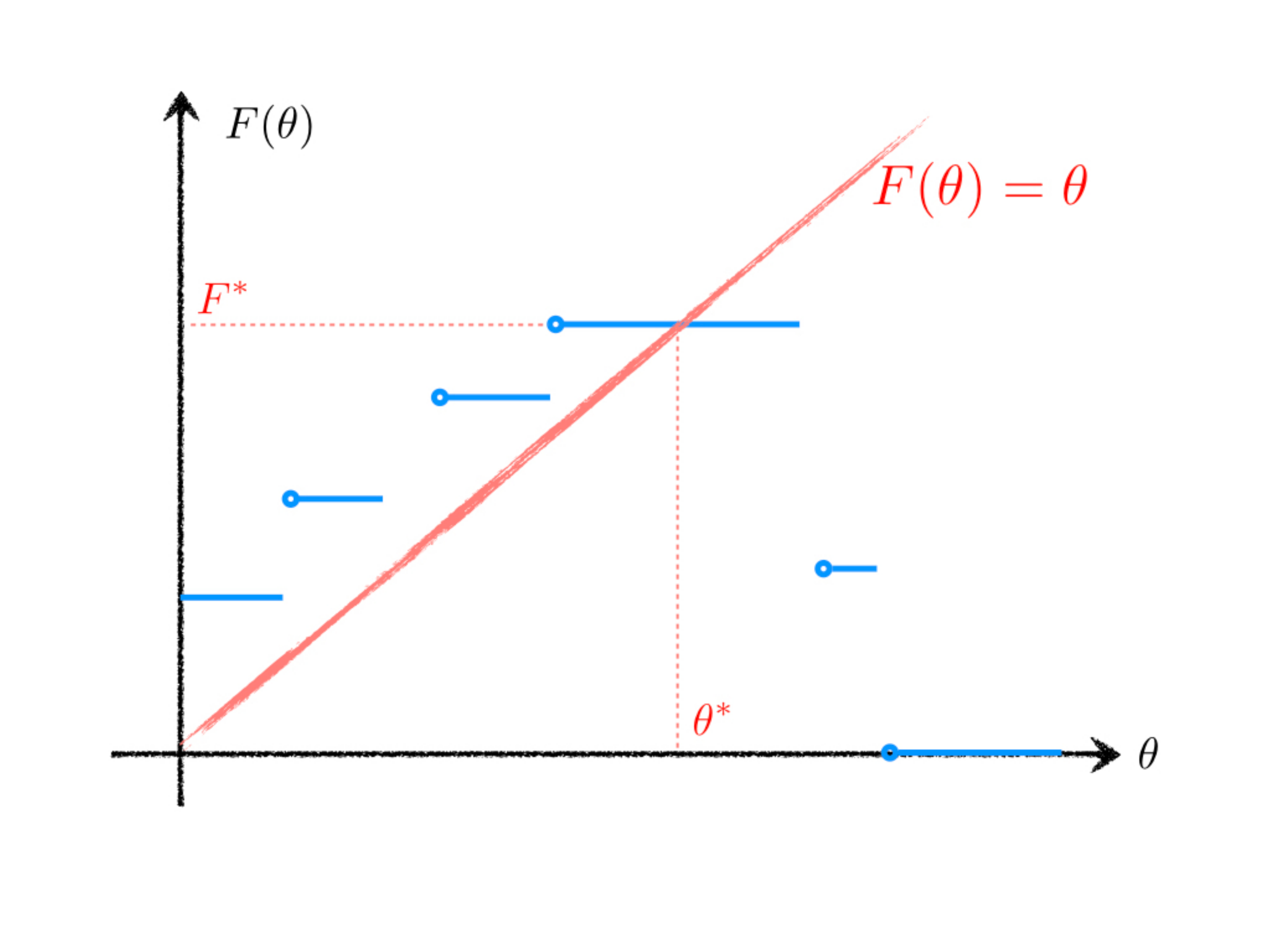}
\caption{\small Illustration of the potential function $F(\theta)$, the important quantities $F^*$ and $\theta^*$, and their properties.}
\label{fig:F}
%\xnote{Need to change $r$, $r^*$ and $F(r)$ to $\theta$ and $\theta^*$ and $F(\theta)$.}
\end{figure}

\section{Trisection and regret analysis}
\label{sec:policy}
We propose an algorithm based on trisections of the potential function $F$ in order to locate level $\theta^*$ at which the maximum expected revenue  $F^*=F(\theta^*)$ is attained.
Our algorithm avoids explicitly estimating individual items' mean utilities $\{v_i\}_{i=1}^N$, and subsequently yields a regret independent of the number of items $N$.
We first give a simplified algorithm (pseudo-code description in Algorithm \ref{alg:trisection}) with an additional $O(\sqrt{\log T})$ term in the regret upper bound and outline its proofs.
We further show how the additional logarithmic   dependency on $T$ can be removed by using more advanced techniques. 
%\xnote{The footnote in Algorithm 1 appears on a different page}
%The complete proofs of all results are deferred to the appendix.
%A pseudo-code description is given in Algorithm \ref{alg:trisection}.

\begin{algorithm}[!t]
	\KwInput{revenue parameters $r_1,\cdots,r_n\in[0,1]$, time horizon $T$}
	\KwOutput{sequence of assortment selections $S_1, S_2, \cdots, S_T\subseteq\mathcal{N}$}
	
	Initialization: $a_0=0$, $b_0=1$\;
	
	\For{$\tau=0,1,\cdots$}{
		$x_{\tau} = \frac{2}{3}a_\tau + \frac{1}{3}b_\tau$, $y_\tau = \frac{1}{3}a_\tau+\frac{2}{3}b_\tau$ \Comment*[r]{\small trisection}
		$\ell_0(x_\tau)=\ell_0(y_\tau)=0$, $u_0(x_\tau)=u_0(y_\tau) = 1$ \Comment*[r]{\small initialization of confidence intervals}
		$\rho_0(x_\tau) = \rho_0(y_\tau) = 0$ \Comment*[r]{\small initialization of accumulated rewards}
		\For{$t=1$ to $16\lceil (y_\tau-x_\tau)^{-2} \ln(T))\rceil$ \textsuperscript{$\dagger$}}{\label{line:alg-mnl-6}
			%   \lIf{$\ell_{t-1}(x_\tau) \leq x_\tau \leq u_{t-1}(x_\tau)$}{
			%     $\rho_t(x_\tau), \ell_t(x_\tau), u_t(x_\tau) \gets \textsc{Explore}(x_\tau,t)$
			%     %\Comment{\small Explore the left midpoint $x_\tau$}
			%   }
			%\lElse{$\rho_t(x_\tau), \ell_t(x_\tau), u_t(x_\tau) \gets \rho_{t-1}(x_\tau), \ell_{t-1}(x_\tau), u_{t-1}(x_\tau)$}
			%\BlankLine
			\lIf{$\ell_{t-1}(y_\tau) \leq y_\tau \leq u_{t-1}(y_\tau)$}{
				%\Comment{\small Explore the right midpoint $y_\tau$}
				$\rho_t(y_\tau), \ell_t(y_\tau), u_t(y_\tau) \gets \textsc{Explore}(y_\tau,t,1/T^2)$
			}\label{line:alg-mnl-7}

			\lElse
			{
				$\rho_t(y_\tau), \ell_t(y_\tau), u_t(y_\tau) \gets \rho_{t-1}(y_\tau), \ell_{t-1}(y_\tau), u_{t-1}(y_\tau)$
			}\label{line:alg-mnl-8}
			%\BlankLine
			
			\BlankLine
			Exploit the left endpoint $a_\tau$: pick assortment $S=\mathcal L_{a_\tau}$\;\label{line:alg-mnl-9}
		}
		
		\Comment{\small Update trisection parameters}
		\lIf{$u_t(y_\tau) < y_\tau$}{ $a_{\tau+1}=a_\tau$, $b_{\tau+1}=y_\tau$ }

		\lElse{$a_{\tau+1} = x_\tau$, $b_{\tau+1}=b_\tau$}
	}
	\vskip 0.1in
	\caption{The trisection algorithm.}
	\label{alg:trisection}
{\nonl \footnotesize \textsuperscript{$\dagger$}Stop whenever the maximum number of iterations $T$ is reached.}
\end{algorithm}
%\footnotetext{}

\setcounter{AlgoLine}{0}
\begin{algorithm}[!t]
	\SetAlgoLined
	\KwInput{revenue level $\theta$, time $t$, confidence level $\delta$}
	\KwOutput{accumulated revenue $\rho_t(\theta)$, confidence intervals $\ell_t(\theta)$ and $u_t(\theta)$}
	Pick assortment $S=\mathcal L_\theta(\mathcal N)$ and observe purchasing action $j\in S\cup\{0\}$\;
	
	Update accumulated reward: $\rho_t(\theta) = \rho_{t-1}(\theta) + r_j$ \Comment*[r]{\small $r_0:=0$}
	Update confidence intervals:\label{line:alg-explore-3}
	$
	%\ell_t(\theta) = \max\left\{\ell_{t-1}(\theta), \frac{\rho_t(\theta)}{t}-\sqrt{\frac{6\ln2T}{t}}\right\};\;\;u_t(\theta)=\min\left\{u_{t-1}(\theta), \frac{\rho_t(\theta)}{t}+\sqrt{\frac{6\ln2T}{t}}\right\}.
	[\ell_t(\theta), u_t(\theta)] = \frac{\rho_t(\theta)}{t} \pm \sqrt{\frac{\log(1/\delta)}{2t}}.
	$
	\vskip 0.1in
	\caption{\textsc{Explore} Subroutine: exploring a certain revenue level $\theta$}
	\label{alg:explore}
\end{algorithm}

To assist with readability, below we list notations used in the algorithm description together with their meanings:
\begin{itemize}[leftmargin=0.2in]
\item[-] $a_\tau$ and $b_\tau$: left and right boundaries that contain $\theta^*$; it is guaranteed that $a_\tau\leq \theta^*\leq b_\tau$ with high probability,
and the regret incurred on failure events is strictly controlled;
\item[-] $x_\tau$ and $y_\tau$: trisection points; $x_\tau$ is closer to $a_\tau$ and $y_\tau$ is closer to $b_\tau$;
\item[-] $\ell_t(y_\tau)$ and $u_t(y_\tau)$: lower and upper confidence bounds for $F(y_\tau)$ established at iteration $t$;
it is guaranteed that $\ell_t(y_\tau)\leq F(y_\tau)\leq u_t(y_\tau)$ with high probability,
and the regret incurred on failure events is strictly controlled;
\item[-] $\rho_t(y_\tau)$: accumulated reward by exploring level set $\mathcal L_{y_\tau}$ up to iteration $t$.
\end{itemize}

With these notations in place, we provide a detailed description of Algorithm \ref{alg:trisection} to facilitate the understanding.
The algorithm operates in epochs (outer iterations) $\tau=1,2,\cdots$  until a total of $T$ assortment selections are made.
The objective of each outer iteration $\tau$ is to find the relative position between trisection points ($x_\tau,y_\tau$)
and the ``reference'' location $\theta^*$, after which the algorithm either moves $a_\tau$ to $x_\tau$ or $b_\tau$ to $y_\tau$, effectively
shrinking the length of the interval $[a_\tau,b_\tau]$ that contains $\theta^*$ to its two thirds.
Furthermore, to avoid a large cumulative regret, level set corresponding to the left endpoint $a_\tau$ is {exploited in each time period within the epoch $\tau$} to offset potentially large regret incurred by exploring $y_\tau$.

%\xnote{Just make sure that $L_{y_\tau}$ is offered for all iterations in epoch $\tau$.}

In Step~\ref{line:alg-mnl-7} and \ref{line:alg-mnl-8} of Algorithm \ref{alg:trisection}, lower and upper confidence bounds $[\ell_t(y_\tau),u_t(y_\tau)]$ for $F(y_\tau)$
are constructed using concentration inequalities (e.g. Hoeffding's inequality \cite{hoeffding1963probability}).
%with adaptively chosen confidence levels $\delta=1/T(y_\tau-x_\tau)^2$.
%Such a choice was motivated by the celebrated MOSS algorithm for multi-armed bandits \cite{audibert2009minimax},
%and our analysis shows that in our trisection algorithmic framework such adaptive confidence choices also leads to the removal of additional $\log T$ terms
 %in the regret upper bound.
%a finite-sample version of the law of iterated logarithms (known as the {lil'UCB} method \cite{jamieson2014lil} under the multi-armed bandit context).
These confidence bounds are updated until the relationship between $y_\tau$ and $F(y_\tau)$ is clear, or a pre-specified number of inner iterations for outer iteration $\tau$
has been reached (set to $n_\tau := \lceil 16(y_\tau-x_\tau)^{-2} \ln(T^2)\rceil$ in Step~\ref{line:alg-mnl-6}). 
%\xnote{In COLT version, we use $n_\tau := \lceil 48 (y_\tau-x_\tau)^{-2} \ln(T^2)\rceil$, please check}
Algorithm \ref{alg:explore} gives detailed descriptions on how such confidence intervals are built, based on repeated exploration of
level set $\mathcal L_{y_\tau}$. %and the fact that the expected revenue collected on $\mathcal L_{y_\tau}$ coincides with $F(y_\tau)$ by definition.

After sufficiently many explorations of $\mathcal L_{y_\tau}$, a decision is made on whether to advance the left bounary (i.e., $a_{\tau+1}\gets x_\tau$)
or the right boundary (i.e., $b_{\tau+1}\gets y_\tau$).
Below we give high-level intuitions on how such decisions are made, with rigorous justifications presented later as part of the proof of the main regret theorem for Algorithm \ref{alg:trisection}.
\begin{enumerate}[leftmargin=0.2in]
\item If there is sufficient evidence that $F(y_\tau) < y_\tau$ (e.g., $u_t(y_\tau)<y_\tau$), then $y_\tau$ must be to the right of $\theta^*$ (i.e., $y_\tau\geq\theta^*$)
due to Lemma \ref{lem:Fmonotonic-right}. Therefore, we will shrink the value of right boundary  by setting $b_{\tau+1}\gets y_\tau$.
\item { On the other hand, when $u_t(y_\tau) \geq y_\tau$, we can conclude that \emph{$x_\tau$ must be to the left of $\theta^*$ (i.e., $x_\tau\leq\theta^*$)}. We show this by contradiction. Assuming that $x_\tau > \theta^*$, since $y_\tau$ is always greater than $x_\tau$ (and thus $y_\tau > \theta^*$) and the gap between $y_\tau$ and $F(y_\tau)$ is at least $y_\tau-x_\tau$ \footnote{By Lemma \ref{lem:Fmonotonic-right}, we have $y_\tau - F(y_\tau)\geq y_\tau-F(x_\tau)\geq y_\tau-x_\tau$}, the gap will be detected by the confidence bounds and thus we will have $u_t(y_\tau) < y_\tau$ with high probability. This leads to a contradiction.
    Since $x_\tau$ is to the left of $\theta^*$, we should increase the value of the left boundary by setting $a_{\tau+1}\gets x_\tau$.
   }
   %\xnote{Yining: I rewrote this part with a bit more details. Please proofread.}
%\item If $u_t(y_\tau)<y_\tau$ fails to hold, then either $F(y_\tau)\leq y_\tau$ or $F(y_\tau)$ is very close to $y_\tau$.
%In either cases, we conclude that \emph{$x_\tau$ must be to the left of $\theta^*$ (i.e., $x_\tau\leq\theta^*$)},
%because otherwise by Lemma \ref{lem:Fmonotonic-right} there will be a gap of width at least $(y_\tau-x_\tau)$ between $y_\tau$ and $F(y_\tau)$, which will be detected by the confidence bounds with high probability. Hence, $a_{\tau+1}\gets x_\tau$ is the correct action.
\end{enumerate}

The following theorem is our main upper bound result for the (worst-case) regret incurred by Algorithm \ref{alg:trisection}.
\begin{theorem}
There exists a universal constant $C_1>0$ such that for all parameters $\{v_i\}_{i=1}^N$ and $\{r_i\}_{i=1}^N$ satisfying $r_i\in[0,1]$,
the regret incurred by Algorithm \ref{alg:trisection} satisfies
\begin{equation}
\reg(\{S_t\}_{t=1}^T) = \mathbb E\sum_{t=1}^TR(S^*)-R(S_t) \leq C_1\sqrt{T\log T}.
\end{equation}
\label{thm:upper}
\end{theorem}

\subsection{Proof sketch}

In the rest of the section we sketch key steps and lemmas towards the proof of Theorem \ref{thm:upper}. The proofs of technical lemmas are provided  in the supplementary material.
We first state a simple lemma showing that the confidence bound $\ell_t(y_\tau)$ and $u_t(y_\tau)$
constructed in Algorithm \ref{alg:trisection} contains $F(y_\tau)$ with high probability.

\begin{lemma}
	With probability $1-O(T^{-1})$, $\ell_t(\theta)\leq F(\theta)\leq u_t(\theta)$ for all $t$.
	\label{lem:CI}
\end{lemma}

The following lemma, based on properties of the potential function $F$ and Lemma \ref{lem:CI},
establishes that (with high probability) the shrinkage of $a_\tau$ or $b_\tau$ are ``consistent'';
i.e., $\theta^*$ is always contained in $[a_\tau,b_\tau]$.
Its proof is based on the intuitive two-case analysis discussed before Theorem \ref{thm:upper} and will be provided in the supplementary material.

\begin{lemma}
With probability $1-O(T^{-1})$, $a_\tau\leq\theta^*\leq b_\tau$ for all $\tau=1,2,\cdots,\tau_0$,
where $\tau_0$ is the last outer iteration of Algorithm \ref{alg:trisection}.
\label{lem:trisection}
\end{lemma}

Using Lemmas \ref{lem:CI} and \ref{lem:trisection}, we are able to prove the following lemma that upper bounds the regret incurred at each outer iteration $\tau$ using the distance between the trisection points $x_\tau$ and $y_\tau$.

\begin{lemma}\label{lem:regret-trisection}
	For $\tau=0,1,\cdots$ let $\mathcal T(\tau)$ denote the set of all indices of inner iterations at outer iteration $\tau$.
	Conditioned on the success events in Lemmas \ref{lem:CI} and \ref{lem:trisection}, it holds that
	\begin{equation}
	\mathbb E\sum_{t\in\mathcal T(\tau)}R(S^*)-R(S_t) \lesssim \varepsilon_\tau^{-1}\log T.
	\end{equation}
\end{lemma}

We are now ready to prove Theorem \ref{thm:upper}. 

%\xnote{can anyone look into latex to change Proof to ``Proof of Theorem \ref{thm:upper}.''}

\begin{proof}
Recall the definition that $\varepsilon_\tau = y_\tau-x_\tau$ for outer iterations $\tau=0,1,\cdots$.
Because after each outer iteration we either set $b_{\tau+1}=y_\tau$ or $a_{\tau+1}=x_\tau$,
it is easy to verify that $\varepsilon_\tau = (2/3)\cdot \varepsilon_{\tau-1}$.
Subsequently, invoking Lemma \ref{lem:trisection} and using summation of geometric series we have
%Summing all three types of regret incurred at each outer iteration $\tau$ we can bound the total regret of Algorithm \ref{alg:trisection}.
%Denote $\tau_0$ as the last outer iteration at which the horizon $T$ is reached and the algorithm stops.
%The overall regret of $\pi$ can then be bounded as
\begin{equation}
\mathbb E\sum_{t=1}^TR(S^*)-R(S_t) \lesssim \sum_{\tau=0}^{\tau_0}{\varepsilon_{\tau}^{-1}\log T}\lesssim \varepsilon_{\tau_0}^{-1}\log T,
\label{eq:di1}
\end{equation}
where $\tau_0$ is the total number of outer iterations executed by Algorithm \ref{alg:trisection}.
On the other hand, because at each outer iteration $\tau$ the revenue level $a_\tau$ is exploited for exactly $n_\tau=16\lceil (y_\tau-x_\tau)^{-2}\ln (T^2)\rceil$ times, we have
\begin{equation}
T \geq n_{\tau_0} \gtrsim \varepsilon_{\tau_0}^{-2}\log T.
\label{eq:di2}
\end{equation}
Combining Eqs.~(\ref{eq:di1}) and (\ref{eq:di2}) we conclude that
$\reg(\{S_t\}_{t=1}^T)  \lesssim \sqrt{T\log T}$.
\end{proof}

%\xnote{We need a bit more details on the proof and thus I extract something from the COLT version. Please check and delete repetitions from the supplement.}

\section{Improved regret with adaptive confidence levels}\label{sec:improved-regret}

%\xnote{needs to be re-wrote.}
In this section we consider a variant of Algorithm \ref{alg:trisection} that achieves an improved regret of $O(\sqrt{T})$.
The key idea is to use an adaptive allocation of confidence levels, by allowing larger failure probability as more data are collected.
This is because  later failures result in smaller accumulated regret.
Such a strategy is motivated by the MOSS algorithm \cite{audibert2009minimax} for multi-armed bandits.
However, our analysis is quite different from \cite{audibert2009minimax},
involving new concentration inequalities and induction arguments tailored specifically to our model and proposed policy.

We start with a new uniform concentration inequality for adaptively chosen confidence levels.
\begin{lemma}
Let $X_1,\cdots,X_L$ be i.i.d.~random variables with mean $\mu$ and satisfy $a\leq X_i\leq b$ almost surely for all $\ell\in[L]$.
For any $\delta\in(0,1]$, it holds that
\begin{equation}
\Pr\left[\forall \ell\in[L], \left|\frac{1}{\ell}\sum_{i=1}^\ell X_i - \mu\right| \leq \sqrt{\frac{2(b-a)^2\ln(8/(\delta\ell))}{\ell}}\right] \geq 1-L\delta.
\end{equation}
\label{lem:uniform-concentration}
\end{lemma}
%\xnote{By Hoeffding's inequality, this should be $\sqrt{\frac{\ln[2/(\delta)] (b-a)^2}{2\ell}}$, do I miss something}

The proof of Lemma \ref{lem:uniform-concentration} is placed in the supplementary material, based on a careful doubling argument
with Hoeffding's maximal inequality (\cite{hoeffding1963probability}, re-phrased in Lemma \ref{lem:hoeffding-maximal}).
Compared to the classical Hoeffding's inequality (Lemma \ref{lem:hoeffding}) with the union bound, 
one notable difference is the increasing ``failure probability'' as $\ell$ increases (effectively $\ell\delta$ in $\sqrt{\frac{2\ln(8/(\delta\ell))(b-a)^2}{\ell}}$  instead of $\delta$). This allows the confidence intervals to be much shorter for large $\ell$.

%The key idea is to use the finite-sample law-of-iterated-logarithm (LIL, \cite{darling1985iterated}) confidence intervals \cite{jamieson2014lil}
%together with an adaptive choice of confidence parameters similar to the MOSS strategy \cite{audibert2009minimax}
%in order to carefully upper bounding regret induced by failure probabilities.

%More specifically, 
With Lemma \ref{lem:uniform-concentration}, we are ready to describe the variant of Algorithm \ref{alg:trisection}, which attains the tight regret bound.
Most steps in Algorithms \ref{alg:trisection} and \ref{alg:explore} remain unchanged,
and the changes are summarized below:
\begin{itemize}[leftmargin=0.2in]
\item[-] Step \ref{line:alg-explore-3} in Algorithm \ref{alg:explore} is replaced with 
\begin{equation}
[\ell_t(\theta), u_t(\theta)] = \frac{\rho_t(\theta)}{t} \pm \sqrt{\frac{2\ln[8/(\delta t)]}{t}}.
\label{eq:lil-ci}
\end{equation}
\item[-] Step \ref{line:alg-mnl-7} in Algorithm \ref{alg:trisection} is replaced with $\textsc{Explore}(y_\tau,t,1/T)$;
%for an adaptive confidence parameter $\delta = 1/(T(y_\tau-x_\tau)^2)$;
correspondingly, the number of inner iterations is changed to $n_\tau = 8\lceil(y_\tau-x_\tau)^{-2}\ln(8T(y_\tau-x_\tau)^2)\rceil$.
\end{itemize}

The first change for improving the regret is the way how confidence intervals $[\ell_t(\theta),u_t(\theta)]$ of $F(\theta)$ is constructed.
Instead of using fixed confidence level $1/T^2$ as in the baseline policy,
in the revised policy \emph{varying} confidence levels are employed, with ``effective'' failure probabilities increase as the algorithm collects  more data.

%Comparing the new confidence interval in Eq.~(\ref{eq:lil-ci}) with the original one in Algorithm \ref{alg:explore},
%the important difference is the $\ln\ln(2T)$ term arising from the law of the iterated logarithm, which makes the confidence intervals hold \emph{uniformly} for all $t$.
%This also leads to a different choice of confidence parameter $\delta$ in constructing confidence intervals, which is the second important change we make.
%In particular, instead of using a universal confidence level
%\footnote{$\delta=O(1/T^2)$ rather than $\delta=O(1/T)$ is used because an additional union bound is required for all inner iterations $t$ in each outer iteration $\tau$ for confidence intervals
%constructed via the Hoeffding's inequality.}
% $\delta = O(1/T^2)$ throughout the entire procedure, ``adaptive'' confidence levels $\delta = O(1/(T(y_\tau-x_\tau)^2))$
%are used, which increases as the algorithm moves onto later iterations.
%Such choice of confidence parameters is motivated by the fact that the accumulated regret suffers less from a confidence interval failure at later iterations. Indeed, since we are relatively closer to the optimal assortment, the ``excess regret'' suffered when the confidence interval fails to cover the true potential function value is smaller.
We also remark that similar confidence parameter choices were also adopted in \cite{audibert2009minimax} to remove additional $\log(T)$ factors
in multi-armed bandit problems.

The following theorem shows that the algorithm variant presented above achieves an asymptotic regret of $O(\sqrt{T})$,
considerably improving Theorem \ref{thm:upper} with an $O(\sqrt{T\log T})$ regret bound.
Its proof is rather technical and involves careful analysis of failure events at each outer iteration $\tau$ of the trisection algorithm.
To highlight the main idea behind the proof, we provide a sketch of the proof in Sec.~\ref{sec:proof_sketch_adap} and defer the entire proof of Theorem \ref{thm:upper-lil} to the supplement.

\begin{theorem}
There exists a universal constant $C_1>0$ such that for all parameters $\{v_i\}_{i=1}^N$ and $\{r_i\}_{i=1}^N$ satisfying $r_i\in[0,1]$,
the regret incurred by the variant of Algorithm \ref{alg:trisection} described above satisfies
\begin{equation}
\mathrm{Regret}(\{S_t\}_{t=1}^T) =\mathbb E\sum_{t=1}^TR(S^*)-R(S_t) \leq C_1\sqrt{T}.
\end{equation}
\label{thm:upper-lil}
\end{theorem}

\subsection{Proof sketch}
\label{sec:proof_sketch_adap}
We sketch key steps and lemmas towards the proof of Theorem \ref{thm:upper}. The proofs of technical lemmas  are provided in the supplementary material.
We first define some notations.
Let $\tau=0,1,\cdots$ be the number of outer iterations in Algorithm \ref{alg:trisection}, $\varepsilon_\tau = (y_\tau-x_\tau)$ be the distance between the two trisection points
at outer iteration $\tau$, and $n_\tau =8\lceil\varepsilon_\tau^{-2}\ln(8T\varepsilon_\tau^2)\rceil$ be
the pre-specified number of inner iterations.
Recall also that $\theta^*=F(\theta^*)=F^*$ is the optimal revenue value suggested by Lemma \ref{lem:rrequal}.

Define the following three disjoint events that partition the entire probabilistic space:
\begin{itemize}[leftmargin=0.2in]
\item {Event $\mathcal E_1(\tau)$}: $\theta^*  < a_\tau < b_\tau$;
\item {Event $\mathcal E_2(\tau)$}: $a_\tau \leq \theta^* \leq b_\tau$;
\item {Event $\mathcal E_3(\tau)$}: $a_\tau < b_\tau < \theta^*$.
\end{itemize}
Let $\tau_0\in\mathbb N$ be the last outer iteration in Algorithm \ref{alg:trisection}.
Let also $\mathcal T(\tau)\subseteq[T]$ be the indices of inner iterations in outer iteration $\tau$,
satisfying $|\mathcal T(\tau)|\leq 2n_\tau$ almost surely.
For $\omega\in\{1,2,3\}$, $\tau\in\mathbb N$ and $\alpha,\beta\in\mathbb R^+$, define
\begin{equation}
\psi_\tau^\omega(\alpha,\beta) := \mathbb E\left[\sum_{\tau'=\tau}^{\tau_0}\sum_{t\in\mathcal T({\tau'})} R(S^*)-R(S_t)\bigg| \mathcal E_\omega(\tau), |a_\tau-\theta^*| = \alpha, |F(a_\tau)-a_\tau|=\beta\right].
\end{equation}
Intuitively, $\psi_\tau^\omega(\alpha,\beta)$ is the expected regret Algorithm \ref{alg:trisection} incurs for outer iterations $\tau,\tau+1,\cdots,\tau_0$,
conditioned on the event 
$\mathcal E_\omega(\tau)$ 
and other boundary conditions at the left margin 
$a_\tau$.

The following three lemmas are the central steps in our proof, which establish recurrence relationships among $\psi_\tau^\omega(\alpha,\beta)$, for $\omega\in\{1,2,3\}$.
The proofs are technically involved and, as we have mentioned, deferred to the supplementary material.
To simplify notations, we write $a_n\lesssim b_n$ or $b_n\gtrsim a_n$ if there exists a \emph{universal} constant $C>0$ such that $|a_n|\leq C|b_n|$ for all $n\in\mathbb N$.

\begin{lemma}[Regret in Case 1]
$\psi_\tau^1(\alpha,\beta) \leq  \beta T + \sum_{\tau'=\tau+1}^{\tau_0} \sup_{\Delta>\varepsilon_{\tau'}} \Delta T\exp\{-n_\tau\Delta^2\} + O(\varepsilon_{\tau'}^{-1}\log(T\varepsilon_{\tau'}^2))$.
\label{lem:case1}
\end{lemma}

\begin{lemma}[Regret in Case 2]
$\psi_\tau^2(\alpha,\beta) \leq O(\varepsilon_\tau^{-1} \log(T\varepsilon_\tau^2))
+ \psi_{\tau+1}^2(\alpha_2',\beta_2')
+ \psi_{\tau+1}^3(\alpha_3',\beta_3')\cdot O(\log(T\varepsilon_\tau^2)/(T\varepsilon_\tau^2)) + \sup_{\Delta>\varepsilon_\tau}\psi_{\tau+1}^1(\alpha_1',\beta_1'(\Delta))\exp\{-n_\tau\Delta_\tau^2\}$
for parameters $\alpha_1',\beta_1'(\Delta),\alpha_2',\beta_2',\alpha_3',\beta_3'$ that satisfy $\beta_1'(\Delta)\leq \Delta$ and $\alpha_3'\leq3\varepsilon_\tau$.
\label{lem:case2}
\end{lemma}

\begin{lemma}[Regret in Case 3]
$\psi_\tau^3(\alpha,\beta) \leq \alpha T$.
\label{lem:case3}
\end{lemma}

We are now ready to complete the proof of Theorem \ref{thm:upper-lil} by combining Lemmas \ref{lem:case2}, \ref{lem:case1} and \ref{lem:case3}.

\begin{proof}
We first get a cleaning expression of $\psi_\tau^1(\alpha,\beta)$ using Lemma \ref{lem:case1}.
First note that $\Delta\mapsto \Delta\exp\{-n_\tau\Delta^2\}$ attains its maximum on $\Delta>0$ at $\Delta=\sqrt{1/2n_\tau}$.
Also note that $n_\tau=\lceil 8\varepsilon_\tau^{-2}\ln(8T\varepsilon_\tau^2)\rceil$ and therefore $\sqrt{1/2n_\tau}\leq\varepsilon_\tau$.
Subsequently,
\begin{align}
\sum_{\tau'=\tau}^{\tau_0}\sup_{\Delta>\varepsilon_\tau}\Delta T\exp\{-n_\tau\Delta^2\}
&\leq \sum_{\tau'=\tau}^{\tau_0} \varepsilon_\tau T\exp\{-n_\tau\varepsilon_\tau^2\}
\leq \sum_{\tau'=\tau}^{\tau_0} \varepsilon_\tau T \exp\{-\ln(T\varepsilon_\tau^2)\}\nonumber\\
&\leq \sum_{\tau'=\tau}^{\tau_0}\varepsilon_\tau^{-1}
= O(\varepsilon_{\tau_0}^{-1}),
\label{eq:delta-extreme}
\end{align}
where the last asymptotic holds because $\{\varepsilon_\tau\}$ forms a geometric series.
Subsequently,
\begin{equation}
\psi_{\tau}^1(\alpha,\beta) \leq \beta T  + \sum_{\tau'=\tau}^{\tau_0}O(\varepsilon_{\tau'}^{-1}\log(T\varepsilon_\tau^2)).
\end{equation}

It remains the bound the summation term on the right-hand side of the above inequality. 
Denote $s_{\tau'} = \varepsilon_{\tau'}^{-1}\ln(T\varepsilon_{\tau'}^2) = \rho^{-\tau'}\ln(T\rho^{2\tau'})$,
where $\rho=2/3$.
We then have $s_{\tau'} = \rho^{\tau_0-\tau'}[1+\ln\rho^{-2(\tau_0-\tau')}] s_{\tau_0} \leq 2(\tau_0-\tau'+1)\rho^{\tau_0-\tau'}\ln(1/\rho)$ for all $\tau'\leq\tau_0$.
Subsequently,
\begin{align}
\sum_{\tau'=\tau}^{\tau_0}s_{\tau'}
&\leq \sum_{\tau'=0}^{\tau_0}2(\tau_0-\tau'+1)\rho^{\tau_0-\tau'}\ln(1/\rho)\cdot s_{\tau_0} \leq O(1)\cdot s_{\tau_0}.
\end{align}
Therefore,
\begin{equation}
\psi_\tau^1(\alpha,\beta) \leq \beta T + O(\varepsilon_{\tau_0}^{-1}\log(T\varepsilon_{\tau_0}^2)).
\label{eq:case1-final}
\end{equation}

We are now ready to derive the final regret upper bound by analyzing $\psi_0^2(\alpha,\beta)$, because the event $\mathcal E_2(0)$ always holds
since $0\leq \theta^*\leq 1$.
Applying Lemma \ref{lem:case2} with Lemma \ref{lem:case3} and Eq.~(\ref{eq:case1-final}), we have for all $\tau\in\{0,1,\cdots,\tau_0\}$ that
\begin{align}
\psi_\tau^2(\alpha,\beta)
&\leq\psi_{\tau+1}^2(\alpha_2',\beta_2') +  O(\varepsilon_\tau^{-1} \log(T\varepsilon_\tau^2)) + O(\varepsilon_\tau T)\cdot \frac{\ln(T\varepsilon_\tau^2)}{T\varepsilon_\tau^2} \nonumber\\
&\;\;\;\;+ \sup_{\Delta>\varepsilon_\tau}\left(\Delta T + O(\varepsilon_{\tau_0}^{-1}\log(T\varepsilon_{\tau_0}^2))\right)\exp\{-n_\tau\Delta^2\}\nonumber\\
&\leq \psi_{\tau+1}^2(\alpha_2',\beta_2') + O(\varepsilon_\tau^{-1} \log(T\varepsilon_\tau^2)) + \sup_{\Delta>\varepsilon_\tau}\Delta T\exp\{-n_\tau\Delta^2\}\nonumber\\
&\;\;\;\;+ O(\varepsilon_{\tau_0}^{-1}\log(T\varepsilon_{\tau_0}^2))\cdot \exp\{-n_\tau\varepsilon_\tau^2\}.
\end{align}
Using the same analysis as in Eq.~(\ref{eq:delta-extreme}), we know $\sup_{\Delta>\varepsilon_\tau} \Delta T\exp\{-n_\tau\Delta^2\} \leq O(\varepsilon_\tau^{-1})$
and $\exp\{-n_\tau\varepsilon_\tau^2\} \leq 1/(T\varepsilon_\tau^2)$.
Subsequently, summing all terms $\tau=0,1,\cdots,\tau_0$ together we have
\begin{align}
\psi_0^2(\alpha,\beta)
&\leq \sum_{\tau=0}^{\tau_0}O(\varepsilon_\tau^{-1}\log(T\varepsilon_\tau^2)) + O(\varepsilon_{\tau_0}^{-1}\log(T\varepsilon_{\tau_0}^2))\cdot \frac{1}{T\varepsilon_\tau^2}\nonumber\\
&\lesssim \varepsilon_{\tau_0}^{-1}\log(T\varepsilon_{\tau_0}^2)\cdot (1 + 1/(T\varepsilon_{\tau_0}^2)).
\end{align}
Finally, note that $n_{\tau_0} \gtrsim \varepsilon_{\tau_0}^{-2}$ and $n_{\tau_0}\leq T$, implying that $\varepsilon_{\tau_0}\gtrsim \sqrt{1/T}$.
Plugging the lower bound on $\varepsilon_{\tau_0}$ into the above inequality we have $\psi_0^2(\alpha,\beta)\lesssim \sqrt{T}$,
which completes the proof of Theorem \ref{thm:upper-lil}.
\end{proof}

%\subsection{Further improvements with doubling checking points}
 
\section{Lower bound}\label{sec:lower}

We prove the following theorem showing that no policy can achieve an accumulated regret smaller than $\Omega(\sqrt{T})$ in the worst case.
\begin{theorem}
Let $N$ and $T$ be the number of items and the time horizon that can be arbitrary.
There exists revenue parameters $r_1,\cdots,r_N\in[0,1]$ such that for any policy $\pi$,
\begin{equation}
\sup_{v_1,\cdots,v_N\geq 0} \reg(\{S_t\}_{t=1}^T)  \geq \frac{\sqrt{T}}{384}.
\end{equation}
\label{thm:lower}
\end{theorem}

Theorem \ref{thm:lower} shows that our regret upper bounds in Theorems \ref{thm:upper} and \ref{thm:upper-lil}
are tight up to $\sqrt{\log T}$ or $\sqrt{\log\log T}$ factors and numerical constants.
We conjecture (in Sec.~\ref{sec:discussion}) that the additional $\sqrt{\log\log T}$ term can also be removed, leading to upper and lower bounds that match up to universal constants.

\subsection{Proof sketch of Theorem \ref{thm:lower}}
We next give a sketch of the proof of Theorem \ref{thm:lower}.
Due to space constraints, we only present an outline of the proof and defer proofs of all technical lemmas to the supplement.
%, deferring some technical lemmas and numerical verifications to the appendix.

We first describe the underlying parameter values on which our lower bound proof is built.
Fix revenue parameters $\{r_i\}_{i=1}^N$ as $r_1=1$, $r_2=1/2$ and $r_3=\cdots=r_N=0$, which are known a priori.
We then consider two constructions of the unknown utility parameters $\{v_i\}_{i=1}^N$:
\begin{eqnarray*}
P_0:& & v_1=1 - 1/4\sqrt{T},\; v_2 = 1,\; v_3=\cdots=v_N=0;\\
P_1:& &v_1=1 + 1/4\sqrt{T},\; v_2 = 1,\; v_3=\cdots=v_N=0.
\end{eqnarray*}
{We note that $P_0$ and $P_1$ also give the probability distributions that characterize the customer random purchasing actions; and thus we will use $P_j[A]$ to denote the probability of event $A$ under the utility parameters specified by $P_j$ for $j\in \{0,1\}$.}

The first lemma shows that there does not exist estimators that can identify $P_0$ from $P_1$ with high probability with only $T$ observations of random purchasing actions.
Its proof involves careful calculation of the Kullback-Leibler (KL) divergence between the two hypothesized distributions
and subsequent application of Le Cam's lemma to the testing question between $P_0$ and $P_1$.
\begin{lemma}
For any estimator $\hat\psi\in\{0,1\}$ whose inputs are $T$ random purchasing actions $i_1,\cdots,i_T$, it holds that
$\max_{j\in\{0,1\}}P_j[\hat\psi\neq j] \geq 1/3$.
\label{lem:minimax}
\end{lemma}

%To prove Lemma \ref{lem:minimax},

On the other hand, the following lemma shows that, if the policy $\pi$ can achieve a small regret under both $P_0$ and $P_1$, then { one can construct an estimator based on $\pi$ such that with large probability the estimator can distinguish between $P_0$ and $P_1$ from observed customers' purchasing actions.}
\begin{lemma}
\label{lem:reduction}
Suppose a policy $\pi$ satisfies $\mathrm{Regret}(\{S_t\}_{t=1}^T) < \sqrt{T}/384$ for both $P_0$ and $P_1$.
Then there exists an estimator $\hat\psi\in\{0,1\}$ such that
$P_{j}[\hat\psi\neq j] \leq 1/4$ for both $j=0$ and $j=1$.
\end{lemma}
Lemma \ref{lem:reduction} is proved by explicitly constructing a classifier (tester) $\hat\psi$ from any sequence of low regret.
In particular, for any assortment sequence $\{S_t\}_{t=1}^T$, we construct $\hat\psi$ as $\hat\psi=0$ if $\frac{1}{T}\sum_{t=1}^T\mathbb I[1\in S_t,2\notin S_t]\geq 1/2$
and $\hat\psi=1$ otherwise.
Using Markov's inequality and the construction of $\{r_i,v_i\}$, it can be shown that if $\mathrm{Regret}(\{S_t\}_{t=1}^T)>\sqrt{T}/384$
then $\hat\psi$ is a good tester with small testing error.
Detailed calculations and the complete proof is deferred to the supplement.

Combining Lemmas \ref{lem:minimax} and \ref{lem:reduction} we proved our lower bound result in Theorem \ref{thm:lower}.

\section{Simulation results}\label{sec:simulation}

We present numerical results of our proposed trisection (and its improved variant) algorithm and compare their performance with several
competitors on synthetic data.

\paragraph{Experimental setup.}
	We generate each of the revenue parameters $\{r_i\}_{i=1}^N$ independently and identically from the uniform distribution on $[.4, .5]$.
	For the preference parameters $\{v_i\}_{i=1}^N$, they are generated independently and identically from the uniform distribution on $[10/N, 20/N]$,
	where $N$ is the total number of items available.

	To motivate our parameter setting, consider the following three types of assortments: %and see why it leads to non-trivial assortment combinations, consider the following three assortments:
	the ``single assortment'' $S=\{i\}$ for some $i\in\mathcal{N}$, the ``full assortment'' $S=\{1,2,\cdots,N\}$,
	and the ``appropriate'' assortment $S=\{i\in\mathcal{N}: r_i\geq 0.42\}$.
	For the single assortment $S=\{i\}$, because the preference parameter for each item is rather small ($v_i\leq 20/N$),
	no single assortment can produce an expected revenue exceeding $0.5\times(20/N)/(1+20/N)= 10/(20+N)$.
	For the full assortment $S=\{1,2,\cdots,N\}$, because $\sum_{i=1}^Nr_iv_i\overset{p}{\to} 0.45\times 15/N\times N = 6.75$ and $\sum_{i=1}^Nv_i\overset{p}{\to} 15$
	by the law of large numbers, the expected revenue of $S$ is around ${6.75}/{(1+15)} = 0.422$.
	Finally, for the ``appropriate'' assortment $S=\{i\in\mathcal{N}:r_i\geq 0.42\}$, we have $\sum_{i\in S}r_iv_i\overset{p}{\to} 0.46\times 15/N\times 0.8N = 5.52$
	and $\sum_{i\in S}v_i \overset{p}{\to} 15/N\times 0.8N = 12$.
	Therefore, the expected revenue of $S$ is around ${5.52}/{(1+12)} = 0.425 > 0.422$.
	The above discussion shows that a revenue threshold $r^*\in(0.4, 0.5)$ is mandatory to extract a portion of the items $\{i\in\mathcal{N}: r_i\geq r^*\}$
	that attain the optimal expected revenue,
	which is highly non-trivial for a dynamic assortment selection algorithm to identify.

\begin{table}[t]
\centering
\caption{Average (mean) and worst-case (max) regret of our trisection (\textsc{Trisec.}) and adaptive trisection (\textsc{Adap-Trisec.}) algorithms and their competitors on synthetic data.
$N$ is the number of items and $T$ is the time horizon.}
\begin{tabular}{l|cccccccccc}
\hline
& \multicolumn{2}{c}{\textsc{Ucb}}& \multicolumn{2}{c}{\textsc{Thompson}}& \multicolumn{2}{c}{\textsc{Grs}}& \multicolumn{2}{c}{\textsc{Trisec.}}& \multicolumn{2}{c}{\textsc{Adap-Trisec.}}\\
$(N,T)$& mean& max& mean& max& mean& max& mean& max& mean& max\\
\hline
%(100,200)& 21.2& 22.4& 1.17& 2.92& 2.89& 13.7& 5.85& 5.85& 4.55& 4.55\\
%(250, 200)& 29.4& 30.5& 2.34& 3.67& 3.83& 16.9& 5.58& 5.58& 4.31& 4.32\\
%(500,200)& 41.6& 42.7& 3.65& 5.16& 4.55& 22.7& 5.58& 5.58& 4.31& 4.31\\
%(1000,200)& 44.3& 45.7& 5.76& 7.05& 2.77& 11.4& 5.66& 5.66& 4.39& 4.39\\
%(2500,200)& 53.3& 54.0& 10.4& 13.9& 1.54& 3.17& 5.63& 5.63& 4.37& 4.37\\
(100,500)& 34.9& 38.1& 1.28& 2.97& 10.9& 22.4& 7.68& 7.68& 1.99& 1.99\\
(250,500)& 54.3& 56.2& 2.81& 4.95& 7.93& 34.2& 7.57& 7.57& 2.23& 2.23\\
(500,500)& 73.4& 75.5& 4.90& 4.95& 7.02& 43.4& 7.43& 7.43& 2.23& 2.23\\
(1000,500)& 90.3& 93.5& 8.17& 10.7& 5.34& 45.1& 7.44& 7.44& 2.25& 2.25\\
\\
(100,1000)& 73.1& 78.2& 1.36& 2.79& 139.9& 175.0& 8.69& 8.69& 3.90& 3.90\\
(250,1000)& 113.7& 119.3& 3.36& 5.17& 90.1& 110.1& 8.69& 8.69& 4.13& 4.14\\
(500,1000)& 136.8& 140.3& 5.65& 7.64& 65.7& 113.9& 9.38& 9.38& 3.80& 3.80\\
(1000, 1000)& 160.8& 165.4& 9.31& 12.4& 8.43& 22.8& 9.77& 9.77& 3.97& 3.97\\
%(2500,1000)& 198.1& 201.5& 18.5& 21.7& 5.45& 38.3& 9.79& 9.79& 6.41& 6.41\\
\hline
\end{tabular}
\label{tab:simulation}
\end{table}

	\paragraph{Comparative methods.}
	%Apart from our trisection algorithm (described in pseudo-code in Algorithm \ref{alg:trisection}, denoted as \textsc{Trisection}),
	Our trisection algorithm with $O(\sqrt{T\log T})$ regret is denoted as \textsc{Trisec}, and its improved adaptive variant (with regret $O(\sqrt{T})$) is denoted as \textsc{Adap-Trisec}.
	The other methods we compare against include the \emph{Upper Confidence Bound} algorithm of \cite{Agrawal16MNLBandit} (denoted as \textsc{Ucb}),
	the \emph{Thompson sampling} algorithm of \cite{Agrawal17Thompson} (denoted as \textsc{Thompson}), and the \emph{Golden Ratio Search} algorithm of \cite{Rusmevichientong2010} (denoted as \textsc{Grs}).
	Note that both \textsc{Ucb} and \textsc{Thompson} proposed in \cite{Agrawal16MNLBandit,Agrawal17Thompson}
	were initially designed for the \emph{capacitated} MNL model, in which the number of items each assortment contains is restricted to be at most $K<N$. In our experiments, we operate both the \textsc{Ucb} and \textsc{Thompson} algorithms under the uncapacitated setting, simply by removing the constraint set when performing each assortment optimization.

	Most hyper-parameters (such as constants in confidence bounds) are set directly using the theoretical values.
	{One exception is our improved adaptive trisection algorithm (\textsc{Adap-Trisec}), 
	in which we replace the $\sqrt{\frac{2\ln(8/(\delta \ell))}{\ell}}$ confidence interval configuration with $\sqrt{\frac{0.1\ln(8/(\delta\ell))}{\ell}}$. We observe that a smaller constant value leads to better empirical performance.%\xnote{where is $(b-a)^2$?}
	}
	%we remove the coefficient of 4 in front of the square root term in the confidence bounds in Eq.~(\ref{eq:lil-ci}),
	%which can be thought of as taking $\varepsilon\to 0^+$ in the finite-sample LIL inequality (see Lemma \ref{lem:lil}) and was also adopted in \cite{jamieson2014lil}.
	Another is the \textsc{Grs} algorithm: in \cite{Rusmevichientong2010} the number of exploration iterations is set to $34\ln(2N)/\beta^2$
	where $\beta=\min_{j\neq j'}|R(\mathcal L_{r_j})-R(\mathcal L_{r_{j'}})|$,
	which is inappropriate for our ``gap-free'' synthetical setting in which $\beta=0$.
	Instead, we use the common choice of $\sqrt{T}$ exploration iterations in typical gap-independent bandit problems for \textsc{Grs}.

	\paragraph{Results.}
	In Table \ref{tab:simulation} we report the mean and maximum regret from 20 independent runs of each algorithm on our synthetic data, with
	different settings of $N$ (number of items) and $T$ (time horizon length).
	We observe that as the number of items ($N$) becomes large, our algorithms (\textsc{Trisec} and \textsc{Adap-Trisec}) achieve smaller mean and maximum regret
	compared to their competitors, and \textsc{Adap-Trisec} consistently outperforms \textsc{Trisec} in all settings.
	Unlike \textsc{Ucb} and \textsc{Thompson} whose regret depend polynomial on $N$,
	our \textsc{Trisec} and \textsc{Adap-Trisec} algorithms have no dependency on $N$ and hence their regret does not increase with $N$.
	{Moreover, the separate exploration and exploitation structure in \textsc{Grs} makes its performance somewhat unstable, which leads to a larger gap between mean and maximum regrets.  }

\section{Conclusion and future directions}\label{sec:discussion}

In this paper we consider the dynamic assortment planning problem under uncapacitated MNL models and derive an optimal regret bound, which is independent of $N$.
%One important open question is to further remove the $O(\sqrt{\log\log T})$ term in the upper bound in Theorem \ref{thm:upper}
%%and eventually achieve upper and lower regret bounds that match each other up to universal numerical constants.
%We conjecture that such improvement is possible by considering a sharper LIL concentration inequality which, instead of holding uniformly for all $t\in\{1,2,\cdots\}$,
%holds only at ``doubling checking'' points $\{1,2,4,8,\cdots\}$.

There are a few interesting future work. In this paper, we assume that  the time horizon length $T$ is known.  It is interesting to design ``horizon-free'' algorithms which adapt to the time horizon $T$.  Moreover, the uncapacitated MNL can be viewed as a capacitated MNL with the capacity upper bound $K=N$. It is known from \cite{Agrawal16MNLBandit} and \cite{chen2018note} that the optimal regret is $\Theta(\sqrt{NT})$ when $K \leq N/4$ and from this paper that the optimal regret is $\Theta(\sqrt{T})$ when $K=N$. It is interesting to investigate the phase transition from $\Theta(\sqrt{NT})$ to $\Theta(\sqrt{T})$. Finally, another direction is to investigate ``instance-optimal'' regret bounds whose regret depends explicitly on the problem parameters $\{r_i\}_{i=1}^n, \{v_i\}_{i=1}^n$ and matching corresponding (instance-dependent) minimax lower bounds in which $\{v_i\}_{i=1}^n$ are known up to permutations. Such instance-optimal regret might potentially depend on ``revenue gaps'' $\Delta_i = R(S^*)-R(\mathcal L_{r_i})$, where $S^*$ is the optimal assortment and $r_i$ is the revenue parameter of the item with the $i$th largest revenue.

\bibliography{refs}
\bibliographystyle{IEEE}

%%%%%%%%%% Merge with supplemental materials %%%%%%%%%%
\pagebreak
%%%%%%%%%% Merge with supplemental materials %%%%%%%%%%
%%%%%%%%%% Prefix a "S" to all equations, figures, tables and reset the counter %%%%%%%%%%
\setcounter{equation}{0}
\setcounter{figure}{0}
\setcounter{table}{0}
\setcounter{page}{1}
\setcounter{section}{0}
\makeatletter
\renewcommand{\theequation}{S\arabic{equation}}
\renewcommand{\thefigure}{S\arabic{figure}}

\renewcommand\thesection{\Alph{section}}
\renewcommand\thesubsection{\thesection.\arabic{subsection}}
\renewcommand\thesubsubsection{\thesubsection.\arabic{subsubsection}}

\title{Supplementary Material for: An Optimal Policy for Dynamic  Assortment Planning Under Uncapacitated Multinomial Logit Models}
\maketitle

This supplementary material provides detailed proofs for technical lemmas whose proofs are omitted in the main text.

\section{Proof of technical lemmas in Sec.~\ref{sec:potential}}

%We first state a simple proposition that outlines the basic properties of the potential function $F$.
%Its verification is easy from the definition and the discretized nature of $F$.
%\begin{proposition}
%There exists $c_0,\cdots,c_m\geq 0$ satisfying $c_i\neq c_{i+1}$ for all $i=0,\cdots,m-1$, and $\mathcal S=\{s_1,\cdots,s_m\}\subseteq \{r_i\}_{i=1}^N$, such that
%\begin{equation}
%F(\theta) = c_0\cdot\mathbb I[\theta\leq s_1] + \sum_{i=1}^{m-1}{c_i\cdot\mathbb I[s_i < \theta\leq s_{i+1}]} + c_m\cdot \mathbb I[\theta>s_m],
%\end{equation}
%where $c_m=0$.
%\end{proposition}

\subsection{Proof of Lemma \ref{lem:rrequal}}

Let $s<s'$ be the two endpoints such that $F(s^+)=F(s')=F^*$ (if there are multiple such $s,s'$ pairs, pick any one of them).
We will prove that $s<F^*\leq s'$, which then implies Lemma \ref{lem:rrequal}.

We first prove $s<F^*$. Assume by contradiction that $F^*\leq s$.
Clearly $s\neq 0$ because $F^*>0$.
By definition of $F$ and $F^*$, we have
\begin{equation}
F^* = F(s') = \frac{\sum_{r_i\geq s'}r_iv_i}{1+\sum_{r_i\geq s'}v_i} \;\;\Longrightarrow\;\; \sum_{r_i\geq s'}(r_i-F^*)v_i = F^*.
\label{eq:rrequal-main}
\end{equation}
Because $F^*\leq s$, adding we have that
\begin{equation}
\sum_{r_i\geq s}(r_i-F^*)v_i\geq F^* \;\;\Longrightarrow\;\; F(s)\geq F^*.
\end{equation}
This contradicts with the fact that $F(s)\neq F(s^+)$ and that $F^*$ is the maximum value of $F$.

We next prove $F^*\leq s'$. Assume by contradiction that $F^*>s'$.
Removing all items corresponding to $r_i=s'$ in Eq.~(\ref{eq:rrequal-main}), we have
\begin{equation}
\sum_{r_i>s'}(r_i-F^*)v_i\geq F^* \;\;\Longrightarrow\;\; F(s'^+) \geq F^*.
\end{equation}
This contradicts with the fact that $F(s'^+)\neq F(s')$ and that $F^*$ is the maximum value of $F$.

\subsection{Proof of Lemma \ref{lem:Fmonotonic-right}}

Because $F(\theta^*)=\theta^*=F^*$ and $F^*$ is the maximum value of $F$, we have $F(\theta)\leq \theta$ for all $\theta\geq \theta^*$.
In addition, for any $\theta\geq \theta^*$, by definition of $F$ we have
\begin{align}
F(\theta)-F(\theta^+) &= R(\{i\in\mathcal N: r_i\geq \theta\}) - R(\{i\in\mathcal N: r_i> \theta\})\\
&= \frac{\sum_{r_i\geq \theta}r_iv_i}{1+\sum_{r_i\geq \theta}v_i} - \frac{\sum_{r_i> \theta}r_iv_i}{1+\sum_{r_i> \theta}v_i}\\
&= \frac{(1+\sum_{r_i>\theta}v_i)(\sum_{r_i\geq \theta}r_iv_i) - (1+\sum_{r_i\geq \theta}v_i)(\sum_{r_i> \theta}r_iv_i)}{(1+\sum_{r_i\geq \theta}v_i)(1+\sum_{r_i> \theta}v_i)}\\
&= \frac{(1+\sum_{r_i> \theta}v_i)(\sum_{r_i=\theta}r_iv_i) - (\sum_{r_i=\theta}v_i)(\sum_{r_i>\theta}r_iv_i)}{(1+\sum_{r_i\geq \theta}v_i)(1+\sum_{r_i> \theta}v_i)}\\
&= \frac{\sum_{r_i=\theta}v_i}{1+\sum_{r_i\geq \theta}v_i}\left[\theta - F(\theta^+)\right].\label{eq:F-diff}
\end{align}
Because $\theta\geq F(\theta)$ holds for all $\theta\geq \theta^*$, we conclude that $\theta\geq F(\theta^+)$ also holds for all $\theta\geq \theta^*$.
Subsequently, the right-hand side of Eq.~(\ref{eq:F-diff}) is non-negative and therefore $F(\theta)\geq F(\theta^+)$.

\subsection{Proof of Lemma \ref{lem:Fmonotonic-left}}

If $F(\theta)\equiv F^*$ for all $\theta\leq \theta^*$ then the lemma clearly holds.
In the rest of the proof we shall assume that there is at least one jumping point strictly smaller than $\theta^*$.
Formally, we let $0<s_1<s_2<\cdots<s_t<\theta^*$ be all jumping points that are strictly smaller than $\theta^*$.
To prove Lemma \ref{lem:Fmonotonic-left}, it suffices to show that $F(s_j)\geq s_j$ and $F(s_j)\geq F(s_j^+)$ for all $j=1,\cdots,t$.

We use induction to establish the above claims.
The base case is $j=t$.
Because $F^*$ is the maximum value of $F$, we conclude that $F(s_t)\leq F^* = F(s_t^+)$.
In addition, because $s_t\leq \theta^*=F^*=F(s_t^+)$, invoking Eq.~(\ref{eq:F-diff}) we have that $F(s_t)\leq F(s_t^+)$.
The base case is then proved.

We next prove the claim for $s_j$, assuming it holds for $s_{j+1}$ by induction.
By inductive hypothesis, $F(s_{j+1})\geq s_{j+1}\geq s_j$. Also, $F(s_j^+)=F(s_{j+1})$ because there is no jump points between $s_j$ and $s_{j+1}$,
and subsequently $F(s_j^+)\geq s_j$. Invoking Eq.~(\ref{eq:F-diff}) we proved $F(s_j)\leq F(s_j^+)$.

To prove $F(s_j)\geq s_j$,
define $\gamma_j := (\sum_{r_i=s_j}v_i)/(1+\sum_{r_i\geq s_j}v_i)$. It is clear that $0\leq \gamma_j \leq 1$.
By Eq.~(\ref{eq:F-diff}), we have
\begin{align}
F(s_j)-s_j &= F(s_j)-F(s_j^+)+F(s_j^+)-s_j\\
&= \gamma_j\left[s_j-F(s_j^+)\right] + F(s_j^+)-s_j\\
&= (1-\gamma_j)\left[F(s_j^+)-s_j\right].
\end{align}
As we have already proved $F(s_j^+)\geq s_j$, the right-hand side of the above inequality is non-negative
and therefore $F(s_j)\geq s_j$.
%The fact that $F(s_j)\leq F(s_j^+)$ also follows by invoking Eq.~(\ref{eq:F-diff}) and note that

\section{Proof of technical lemmas in Sec.~\ref{sec:policy}}

%We first prove two technical lemmas showing that with high probability, the confidence intervals $[\ell_t(\theta),u_t(\theta)]$ constructed in Algorithm \ref{alg:explore} contains the true parameter $F(\theta)$, and the optimal revenue level $\theta^*$ is contained in $[a_\tau,b_\tau]$ for all $\tau$.

\subsection{Proof of Lemma \ref{lem:CI}}
Let $\delta=1/T^2$ be the confidence parameter in Algorithm \ref{alg:explore}. By Hoeffding's inequality (Lemma \ref{lem:hoeffding}) and the fact that $0\leq F(\theta)\leq 1$ for all $\theta$,
we have
\begin{align}
\Pr\left[F(\theta)\notin[\ell_t(\theta), u_t(\theta)]\right]
&= \Pr\left[\left|\frac{\rho_t(\theta)}{t}-F(\theta)\right| > \sqrt{\frac{\ln(1/\delta)}{2t}}\right]\\
&\leq 2\exp\left\{-2t\cdot {\frac{\ln(1/\delta)}{2t}}\right\} \leq 2\delta = 2/T^2.
\end{align}
Subsequently, by union bound the probability of $F(\theta)\notin[\ell_t(\theta),u_t(\theta)]$ for \emph{at least} one $t$
is at most $O(T^{-1})$.

%\begin{lemma}
%With probability $1-O(T^{-1})$, $a_\tau\leq\theta^*\leq b_\tau$ for all $\tau=1,2,\cdots,\tau_0$,
%where $\tau_0$ is the last outer iteration of Algorithm \ref{alg:trisection}.
%\label{lem:trisection}
%\end{lemma}
\subsection{Proof of Lemma \ref{lem:trisection}}
We use induction to prove this lemma. We also conditioned on the fact that $\ell_t(x_\tau)\leq F(x_\tau)\leq u_t(x_\tau)$
and $\ell_t(y_\tau)\leq F(y_\tau)\leq u_t(y_\tau)$ for all $t$ and $\tau$, which happens with probability at least $1-O(T^{-1})$
by Lemma \ref{lem:CI}.

We first prove the lemma for the base case of $\tau=0$.
According to the initialization step in Algorithm \ref{alg:trisection}, we have $a_\tau=0$ and $b_\tau=1$.
On the other hand, for any $\theta\geq 0$ it holds that $0\leq F(\theta)\leq F^*\leq 1$.
Therefore, $0\leq \theta^*\leq 1$ and hence $a_\tau\leq \theta^*\leq b_\tau$ for $\tau=0$.

We next prove the lemma for outer iteration $\tau$, assuming the lemma holds for outer iteration $\tau-1$ (i.e., $a_{\tau-1}\leq r^*\leq b_{\tau-1}$).
According to the trisection parameter update step in Algorithm \ref{alg:trisection}, the proof can be divided into two cases:

\bigskip
\noindent\emph{Case 1: $u_t(y_{\tau-1})<y_{\tau-1}$}.
Because $\ell_t(y_{\tau-1})\leq F(y_{\tau-1})\leq u_t(y_{\tau-1})$ always holds, we conclude in this case that $F(y_{\tau-1})<y_{\tau-1}$.
Invoking Lemma \ref{lem:Fmonotonic-left} we conclude that $b_\tau = y_{\tau-1}>\theta^*$.
On the other hand, by inductive hypothesis $a_\tau=a_{\tau-1}\leq \theta^*$.
Therefore, $a_\tau\leq r^*\leq b_\tau$.

\bigskip
\noindent\emph{Case 2: $u_t(y_{\tau-1})\geq y_{\tau-1}$}.
In this case, the revenue level $y_{\tau-1}$ must be explored at every inner iteration in Algorithm \ref{alg:trisection} at outer iteration $\tau-1$,
because $u_t(y_{\tau-1})$ is a non-increasing function of $t$.
Denote $\varepsilon_\tau =y_\tau-x_\tau$ and $n_\tau=16\lceil \varepsilon_\tau^{-2}\ln(T^2)\rceil$ as the number of inner iterations in outer iteration $\tau$.
Subsequently, the length of the confidence intervals on $y_{\tau-1}$ at the end of all inner iterations can be upper bounded by
\begin{equation}
\left|u_t(y_{\tau-1})-\ell_t(y_{\tau-1})\right| \leq 2\sqrt{\frac{\ln(T^2)}{n_\tau}} \leq \frac{1}{2}\varepsilon_\tau^{-1}.
\label{eq:CI-length}
\end{equation}
%Here the last inequality holds because $n_\tau\geq 48(y_{\tau-1}-x_{\tau-1})^{-2}\ln(2T)$.
Invoking Lemma \ref{lem:CI} we then have
\begin{equation}
F(y_{\tau-1}) \geq \ell_t(y_{\tau-1}) \geq u_t(y_{\tau-1}) -  \frac{y_{\tau-1}-x_{\tau-1}}{{2}} \geq y_{\tau-1} - \frac{y_{\tau-1}-x_{\tau-1}}{{2}}.
\label{eq:trisection-case2}
\end{equation}

We now establish that $x_{\tau-1}\leq \theta^*$, which implies $a_\tau\leq \theta^*\leq b_\tau$ because $a_\tau=x_{\tau-1}$ and $b_\tau=b_{\tau-1}\geq \theta^*$ by the inductive hypothesis.
Assume by contradiction that $x_{\tau-1}>\theta^*$.
By Lemma \ref{lem:Fmonotonic-right}, $F(x_{\tau-1})\leq x_{\tau-1}$ and $F(x_{\tau-1})\geq F(y_{\tau-1})$.
Subsequently,
\begin{equation}
F(y_{\tau-1}) \leq x_{\tau-1} = y_{\tau-1} - (y_{\tau-1}-x_{\tau-1}) < y_{\tau-1} - \frac{y_{\tau-1}-x_{\tau-1}}{{2}},
\end{equation}
which contradicts Eq.~(\ref{eq:trisection-case2}).

\subsection{Proof of Lemma \ref{lem:regret-trisection}}
This  lemma upper bounds the expected regret incurred at each outer iteration $\tau$, conditioned on the success events in Lemmas \ref{lem:CI} and \ref{lem:trisection}.

We analyze the regret incurred at outer iteration $\tau$ from exploration of $y_\tau$ and exploitation of $a_\tau$ separately.
\begin{enumerate}[leftmargin=*]
\item \emph{Regret from exploring $y_\tau$}: suppose the level set $\mathcal L_{y_\tau}(\mathcal N)$ is explored for $m_\tau\leq n_\tau$ times at outer iteration $\tau$.
Then we have $u_{m_\tau}(y_\tau) \geq y_\tau$.
In addition, by Lemma \ref{lem:CI} and widths in the constructed confidence bands $\ell_{m_\tau}(y_\tau)$ and $u_{m_\tau}(y_\tau)$,
 we have with probability $1-O(T^{-1})$ that $\ell_{m_\tau}(y_\tau)\leq F(y_\tau)\leq u_{m_\tau}(y_\tau)$
and $|u_{m_\tau}(y_\tau)-\ell_{m_\tau}(y_\tau)|\leq 2\sqrt{(\ln (T^2)/2m_\tau}$.
%Note that the special case of $t'-1=0$ is handled by noting that $|u_0(y_\tau)-\ell_0(y_\tau)|\leq 1$.
Subsequently,
\begin{align}
F(y_\tau)
&\geq \ell_{m_\tau}(y_\tau) \geq u_{m_\tau}(y_\tau) - 2\sqrt{\frac{\ln (T^2)}{2m_\tau}}\geq y_\tau -  2\sqrt{\frac{\ln T}{m_\tau}}.
\end{align}
Note also that $y_\tau \geq a_\tau \geq \theta^*-3\varepsilon_\tau = F^*-3\varepsilon_\tau$; we have
\begin{equation}
F^*-F(y_\tau) \leq3\varepsilon_\tau + 2\sqrt{\frac{\ln T}{m_\tau}}.
\end{equation}

By Lemma \ref{lem:rrequal}, $F^*=R(S^*)$ and therefore the right-hand side of the above inequality is an upper bound on the regret incurred by exploring revenue level $y_\tau$ (corresponding to
the assortment selection $\mathcal L_{y_\tau}$) once.
As the exploration is carried out for $m_\tau$ times, the total regret for all exploration steps at revenue level $x_\tau$ can be upper bounded by
\begin{align}
m_\tau\left[3\varepsilon_\tau + 2\sqrt{\frac{\ln T}{m_\tau}}\right]
\leq 3m_\tau\varepsilon_\tau + \sqrt{4m_\tau\ln T}\leq 3n_\tau\varepsilon_\tau + \sqrt{4n_\tau\ln T}
\lesssim \varepsilon_\tau^{-1}\log T.
\end{align}
Here the last inequality holds because $n_\tau\leq 16\varepsilon_\tau^{-2}\ln(T^2)$.

%\item \emph{Regret from exploring $y_\tau$}: this case is symmetric to the $x_\tau$ explorations.
%Thus, using the same analysis the regret for these explorations can be upper bounded by $O(\varepsilon_\tau^{-1}\log T)$.

\item \emph{Regret from exploiting $a_\tau$}: by Lemma \ref{lem:trisection}, $a_\tau\leq \theta^*$, and therefore $F(a_\tau)\geq a_\tau$.
In addition, $a_\tau\geq \theta^*-3\varepsilon_\tau$ by the definition of $\varepsilon_\tau$.
Subsequently,
\begin{equation}
F(a_\tau) \geq a_\tau \geq \theta^*-3\varepsilon_\tau = F^*-3\varepsilon_\tau.
\label{eq:exploit1}
\end{equation}
Re-organizing terms on both sides of the above inequality and noting that $F^*=F(S^*)$, we have
\begin{equation}
F(S^*)-F(a_\tau) \leq 3\varepsilon_\tau.
\label{eq:exploit2}
\end{equation}
Therefore, the regret for each exploitation of revenue level $a_\tau$ (corresponding to the assortment selection $\mathcal L_{a_\tau}$)
can be upper bounded by $\varepsilon_\tau$.
Because the revenue level $a_\tau$ is exploited for $n_\tau$ times and $n_\tau\leq {16\varepsilon_\tau^{-2}\ln(T^2)}$,
the total regret of exploitation of $a_\tau$ at outer iteration $\tau$ can be upper bounded by
\begin{equation}
n_\tau\cdot 3\varepsilon_\tau \lesssim\varepsilon_\tau^{-1}\log T.
\label{eq:exploit3}
\end{equation}
\end{enumerate}

\section{Proof of technical lemmas in Sec.~\ref{sec:improved-regret}}

\subsection{Proof of Lemma \ref{lem:uniform-concentration}}

Without loss of generality we assume $X_1,\cdots,X_L\in[0,1]$ almost surely, while the general case of $X_1,\cdots,X_L\in[a,b]$ can be dealt with
by a simple re-scaling argument.
Denote $k := \lfloor\log_2 L\rfloor$. For each $\ell\in\{1,2,4,\cdots, 2^k\}$, 
by standard Hoeffding's inequality (Lemma \ref{lem:hoeffding}), we have
\begin{equation*}
\Pr\left[\left|\frac{1}{\ell}\sum_{i=1}^\ell X_i - \mu\right|\leq \sqrt{\frac{\ln[8/(\delta\ell)]}{2\ell}}\right] \geq 1-\frac{\delta\ell}{4}.
%\label{eq:proof-uniform-eq1}
\end{equation*}
Subsequently, by union bound and the fact that $1+2+4+\cdots+2^k\leq 2^{k+1}\leq 2L$, we have
\begin{equation}
\Pr\left[\forall \ell=1,2,4,\cdots,2^k, \left|\frac{1}{\ell}\sum_{i=1}^\ell X_i - \mu\right|\leq \sqrt{\frac{\ln[8/(\delta\ell)]}{2\ell}}\right] \geq 1-\frac{\delta L}{2}.
\label{eq:proof-uniform-eq1}
\end{equation}

Next consider any $\ell\in\{1,2,4,\cdots,2^k\}$.
By Hoeffding's maximal inequality (Lemma \ref{lem:hoeffding-maximal}), we have 
\begin{equation*}
\Pr\left[\forall i\leq \min\{\ell, n-\ell\}, \big|X_{\ell+1}+\cdots + X_{\ell+i} - i\cdot\mu\big| \leq \sqrt{\frac{\ell}{2}\ln[8/(\delta\ell)]}\right] \geq 1-\frac{\delta\ell}{4}.
\end{equation*}
Again using union bound over all $\ell=1,2,4,\cdots,2^k$ and the fact that $1+2+4+\cdots+2^k\leq 2^{k+1} \leq 2L$, we have
\begin{equation}
\Pr\left[\forall \ell=1,2,\cdots,2^k, i\leq \min\{\ell, n-\ell\}, \big|X_{\ell+1}+\cdots + X_{\ell+i} - i\cdot\mu\big| \leq \sqrt{\frac{\ell}{2}\ln[8/(\delta\ell)]}\right] \geq 1-\frac{\delta L}{2}.
\label{eq:proof-uniform-eq2}
\end{equation}

Combining Eqs.~(\ref{eq:proof-uniform-eq1},\ref{eq:proof-uniform-eq2}), we have with probability $1-\delta L$ 
uniformly over all $\ell=1,2,4,\cdots, 2^k$ and $i\leq\min\{\ell,n-\ell\}$ that 
\begin{equation*}
\big|X_1+\cdots+X_\ell+X_{\ell+1}+\cdots+X_{\ell+i}-(\ell+i)\mu\big| \leq \sqrt{2\ell\ln[8/(\delta\ell)]}.
\end{equation*}
Dividing both sides of the above inequality by $(\ell+i)$ we complete the proof of Lemma \ref{lem:uniform-concentration}.

\subsection{Proof of Lemma \ref{lem:case2}}

First analyze the expected regret incurred at outer iteration $\tau$.
by exploiting the left end-point $a_\tau$ (corresponding to assortment $\mathcal L_{a_\tau}$) for $n_\tau$ iterations.
Also, because $a_\tau\leq\theta^*\leq b_\tau$ conditioned on $\mathcal E_2(\tau)$,
by Lemmas \ref{lem:rrequal} and \ref{lem:Fmonotonic-left}
we have $F(a_\tau) \geq a_\tau \geq \theta^*-|b_\tau-a_\tau| = F(\theta^*)-|b_\tau-a_\tau| \geq R(S^*)-3\varepsilon_\tau$.
Subsequently,
\begin{equation}
\text{\it Regret by exploiting $\mathcal L_{a_\tau}$:}\;\;\;\; \leq 3\varepsilon_\tau\cdot n_\tau \lesssim \varepsilon_\tau^{-1} \log(T\varepsilon_\tau^2).
\label{eq:regret-1}
\end{equation}

Next we analyze the expected regret incurred at outer iteration $\tau$ by exploring
the right trisection point $y_\tau$ (corresponding to assortment $\mathcal L_{y_\tau}$).
This is done by a case analysis.
If $y_\tau\leq \theta^*$, then the regret incurred by exploiting $\mathcal L_{y_\tau}$ at outer iteration $\tau$
is again upper bounded (up to numerical constants) by $\varepsilon_\tau^{-1}\log(T\varepsilon_\tau^2)$, similar to Eq.~(\ref{eq:regret-1}).
Otherwise, for the case of $y_\tau>\theta^*$, define $\Delta_\tau := y_\tau - F(y_\tau)$.
By Lemma \ref{lem:Fmonotonic-right}, we know $\Delta_\tau\geq 0$,
and also by Lemma \ref{lem:rrequal}, each exploration of $\mathcal L_{y_\tau}$ incurs a regret of no more than $\Delta_\tau$.
Let $m_\tau$ be the number of times $\mathcal L_{y_\tau}$ is explored at outer iteration $\tau$.
By definition of the stopping rule in Algorithm \ref{alg:trisection}, we have
\begin{align}
\Pr\left[m_\tau \geq \ell\right]
&\leq \Pr\left[\frac{\rho_\ell}{\ell} + \sqrt{\frac{2\ln(8T/\ell)}{\ell}} \geq y_\tau\right]\nonumber\\
&= \Pr\left[\frac{\rho_\ell}{\ell} -F(y_\tau) \geq \Delta_\tau - \sqrt{\frac{2\ln(8T/\ell)}{\ell}}\right].
\end{align}
Because $\rho_\ell$ is a sum of $\ell$ i.i.d.~random variables with mean $F(y_\tau)$ and values in $[0,1]$ almost surely,
applying Hoeffding's inequality (Lemma \ref{lem:hoeffding}) we have
%applying a finite-sample version of the law-of-iterated-logarithm (Lemma \ref{lem:lil}) we have
\begin{align*}
\Pr\left[m_\tau\geq \ell\right]
&\leq \exp\left\{-2\left(\sqrt{\ell}\Delta_\tau - \sqrt{2\ln(8T/\ell)}\right)^2\right\}\nonumber\\
&\lesssim \left\{\begin{array}{ll}
1,& \text{if } \Delta_\tau \leq \sqrt{8\ln(8T/\ell)/\ell};\\
\exp\{-\ell\Delta_\tau^2/2\},& \text{otherwise}.
\end{array}
\right.
\end{align*}
Subsequently,
\begin{align}
&\text{\it Regret by exploring $\mathcal L_{y_\tau}$:}\;\;\;
\leq \sum_{\ell=1}^{n_\tau} \ell\Delta_\tau\Pr[m_\tau = \ell] \leq \sum_{\ell=1}^{n_\tau}\Delta_\tau\Pr[m_\tau\geq \ell]\nonumber\\
&\lesssim \sum_{\ell=1}^{\ell_0-1}\sqrt{\frac{\ln(T/\ell)}{\ell}} + \sum_{\ell=\ell_0}^{n_\tau}\Delta_\tau\exp\{-\ell\Delta_\tau^2/2\}\label{eq:key1-case2}\\
&\lesssim \sqrt{\ell_0\ln(T/\ell_0)} + \sup_{\Delta>\sqrt{8\ln(8T/\ell_0)/\ell_0}} \Delta\cdot \sum_{\ell=\ell_0}^{\infty}\exp\{-\ell\Delta^2/2\}\label{eq:key2-case2}\\
&\leq \sqrt{\ell_0\ln(T/\ell_0)}  + \sup_{\Delta>\sqrt{8\ln(8T/\ell_0)/\ell_0}} \frac{\Delta\exp\{-\ell_0\Delta^2/2\}}{1-\exp\{-\Delta^2/2\}}\nonumber\\
&\leq \sqrt{\ell_0\ln(T/\ell_0)}  + \sup_{\Delta>\sqrt{8\ln(8T/\ell_0)/\ell_0}} \frac{\Delta\exp\{-\Delta^2/2\}}{1-\exp\{-\Delta^2/2\}}\nonumber\\
&\leq \sqrt{\ell_0\ln(T/\ell_0)} +  \sqrt{\frac{8\ln(8T/\ell_0)}{\ell_0}}\cdot \frac{1}{1-\exp\{-4\ln(8T/\ell_0)\}}\label{eq:key3-case2}\\
&\lesssim \sqrt{\ell_0\ln(T/\ell_0)}.\label{eq:regret-2}
%&\leq \sum_{\ell=0}^{n_\tau} \ell\Delta_\tau\cdot \Pr\left[m_\tau\geq\ell\right]\nonumber\\
%&\leq \min_{0\leq \ell_0\leq n_\tau} \left\{\ell_0\Delta_\tau + \sum_{\ell=\ell_0+1}^{n_\tau}\ell\Delta_\tau\Pr\left[m_\tau\geq\ell\right]\right\}\nonumber\\
%&\lesssim \ell_0^*\sqrt{\frac{\log(T\varepsilon_\tau^2)}{\ell_0^*}} + \sum_{\ell_0=\ell_0^*+1}^{n_\tau}\ell\Delta_\tau\exp\{-8\ell\Delta_\tau^2\}\label{eq:key1-case2}\\
%&\leq \sqrt{n_\tau\log(T\varepsilon_\tau^2)} + \int_0^\infty \Delta_\tau xe^{-8\Delta_\tau^2x}\ud x\nonumber\\
%&\lesssim
\end{align}
Here in Eq.~(\ref{eq:key1-case2}), $\ell_0$ is the smallset positive integer not exceeding $n_\tau$ such that $\Delta_\tau> \sqrt{8\ln(8T/\ell_0)/\ell_0}$.
(If $\Delta_\tau\leq\sqrt{8\ln(8T/\ell_0)/\ell_0}$ holds for all $1\leq \ell_0\leq n_\tau$, then the second term in Eq.~(\ref{eq:key1-case2}) is 0 and one can conveniently set $\ell_0=n_\tau+1$
in this case.);
Eq.~(\ref{eq:key2-case2}) holds because
\begin{equation*}
\sum_{\ell=1}^{\ell_0}\sqrt{\frac{\ln(T/\ell)}{\ell}} 
\leq \sum_{j=1}^{\lceil\log_2\ell_0\rceil} 2^j\sqrt{\frac{\ln(T/2^j)}{2^j}}
= \sum_{j=1}^{\lceil\log_2\ell_0\rceil} \sqrt{2^j\ln(T/2^j)}
\lesssim \sqrt{\ell_0\ln(T/\ell_0)};
\end{equation*}
Eq.~(\ref{eq:key3-case2}) holds because $\Delta\mapsto \Delta e^{-\Delta^2/2}/(1-e^{-\Delta^2/2})$ is monotonically decreasing on $\Delta>0$.
Finally, because $\ell_0\leq n_\tau$ and $n_\tau\lesssim \varepsilon_\tau^{-2}\log(T\varepsilon_\tau^2)\geq \varepsilon_\tau^{-2}$, we have
\begin{equation}
\text{\it Regret by exploring $\mathcal L_{y_\tau}$} \lesssim \sqrt{n_\tau\ln(T/n_\tau)} \lesssim \varepsilon_\tau^{-1}\log(T\varepsilon_\tau^{2}).
\label{eq:regret-2-reformulate}
\end{equation}

Finally, we consider regret incurred at later outer iterations $\tau'=\tau+1,\cdots,\tau_0$.
This is done by another case analysis on the relative location of $\theta^*$ with respect to $a_{\tau+1}$ and $b_{\tau+1}$:
\begin{itemize}[leftmargin=0.2in]
\item[-] $\mathcal E_2(\tau+1)$: $a_{\tau+1}\leq\theta^*\leq b_{\tau+1}$: the additional regret is upper bounded by $\psi_{\tau+1}^2(\alpha_1',\beta_1')$ for
some values of $\alpha_1',\beta_1'$ that are not important;
\item[-] $\mathcal E_1(\tau+1)$: $\theta^* < a_{\tau+1} < b_{\tau+1}$: the additional regret is upper bounded by $\psi_{\tau+1}^1(\alpha_2',\beta_2')$
with $\beta_2'\leq \Delta_\tau=y_\tau-F(y_\tau)$ and the value of $\alpha_2'$ not important;
\item[-] $\mathcal E_3(\tau+1)$: $a_{\tau+1}<b_{\tau+1}<\theta^*$: the additional regret is upper bounded by $\psi_{\tau+1}^3(\alpha_3',\beta_3')$ with $\alpha_3'\leq 3\varepsilon_\tau$ and
the value of $\beta_3'$ not important.
\end{itemize}

It remains to upper bound the probability the latter two cases above occur.
$\mathcal E_1(\tau+1)$ occurs if for all inner iterations $t\in\mathcal T(\tau)$, the exploration step fails to detect $F(y_\tau)$ below $y_\tau$,
meaning that $\frac{\rho_\ell}{\ell} + \sqrt{\frac{2\ln(8T/\ell)}{\ell}} > y_\tau$ for all $\ell\in\{1,\cdots,n_\tau\}$.
Also note that because $\theta^*<a_{\tau+1}=x_\tau=y_\tau-\varepsilon_\tau$, by Lemma \ref{lem:Fmonotonic-right} we know that $\Delta_\tau = y_\tau-F(y_\tau) \geq \varepsilon_\tau$.
Using Hoeffding's inequality, we have
\begin{align}
\Pr[\mathcal E_1(\tau+1)]
&\leq \Pr\left[\forall \ell, \frac{\rho_\ell}{\ell} - F(y_\tau) > \Delta_\tau -\sqrt{\frac{2\ln(8T/\ell)}{\ell}}\right]\nonumber\\
&\leq \Pr\left[\frac{\rho_{n_\tau}}{n_\tau} - F(y_\tau) > \Delta_\tau - \sqrt{\frac{2\ln(8T/\ell)}{n_\tau}}\right]\nonumber\\
&\leq \exp\left\{-2\left(\sqrt{n_\tau}\Delta_\tau - \sqrt{2\ln(8T/n_\tau)}\right)^2\right\}\nonumber\\
%&\leq \exp\{-2n_\tau\Delta_\tau^2 + 4\ln(T\varepsilon_\tau^2)\}\nonumber\\
%&\leq \exp\{-(8\varepsilon_\tau^{-2}\Delta_\tau^2-4)\ln(T\varepsilon_\tau^2)\}.
&\leq \exp\left\{-n_\tau\Delta_\tau^2\right\}.
\label{eq:right-trisection-prob}
%&\leq \exp\{-4\ln(T\varepsilon_\tau^2)\} \leq (T\varepsilon_\tau^2)^{-4}.
%&\lesssim \sqrt{\frac{\ln(T\varepsilon_\tau^2)}{n_\tau}}\cdot\vct 1\left[\Delta_\tau\leq 4\sqrt{\frac{\ln(T\varepsilon_\tau^2)}{n_\tau}}\right] + \exp\{-8n_\tau\Delta_\tau^2\}\cdot\vct 1\left[\Delta_\tau> 4\sqrt{\frac{\ln(T\varepsilon_\tau^2)}{n_\tau}}\right]\nonumber\\
%&\lesssim \varepsilon_\tau^{-1}\log(T\varepsilon_\tau^2)\cdot\vct 1\left[\Delta_\tau\leq 4\sqrt{\frac{\ln(T\varepsilon_\tau^2)}{n_\tau}}\right] \nonumber\\
%&\;\;\;\;+ \exp\{-8\varepsilon_\tau^{-2}\ln(T\varepsilon_\tau^2)\Delta_\tau^2\}\cdot\vct 1\left[\Delta_\tau> 4\sqrt{\frac{\ln(T\varepsilon_\tau^2)}{n_\tau}}\right].
\end{align}
Here Eq.~(\ref{eq:right-trisection-prob}) holds because $\sqrt{n_\tau}\Delta_\tau \geq \sqrt{n_\tau}\varepsilon_\tau\geq \sqrt{8\ln(8T\varepsilon_\tau^2)} \geq 2\sqrt{2\ln(8T/n_\tau)}$
by the choice of $n_\tau$.

The $\mathcal E_3(\tau+1)$ event occurs if the exploration step in Algorithm \ref{alg:trisection} falsely detects $y_\tau>F(y_\tau)$ at some stage $\ell\in\{1,\cdots,n_\tau\}$,
meaning that $\frac{\rho_\ell}{\ell}+\sqrt{\frac{2\ln(8T/\ell)}{\ell}} < y_\tau$.
Note that because $b_{\tau+1}=y_\tau<\theta^*$, by Lemma \ref{lem:Fmonotonic-left}, we know $F(y_\tau)\geq y_\tau$.
%By the law of iterated logarithm (Lemma \ref{lem:lil}) we have
By Lemma \ref{lem:uniform-concentration},
\begin{align}
\Pr[\mathcal E_3(\tau+1)]
&= \Pr\left[\exists\ell, \frac{\rho_\ell}{\ell} < y_\tau - \sqrt{\frac{2\ln(8T/\ell)}{\ell}}\right]\nonumber\\
&\leq \Pr\left[\exists\ell,\left|\frac{\rho_\ell}{\ell} -F(y_\tau)\right|> \sqrt{\frac{2\ln(8T/\ell)}{\ell}}\right]\nonumber\\
&\lesssim \frac{n_\tau}{T} \lesssim \frac{\ln(T\varepsilon_\tau^2)}{T\varepsilon_\tau^2}.
%&= \Pr\left[\exists\ell, \frac{\rho_\ell}{\ell}-F(y_\tau)< -\sqrt{\frac{2\ln(T\varepsilon_\tau^2)}{\ell}}\right]\nonumber\\
%&\leq \sum_{\ell=1}^{n_\tau} \Pr\left[\frac{\rho_\ell}{\ell}-F(y_\tau)< -\sqrt{\frac{2\ln(T\varepsilon_\tau^2)}{\ell}}\right]\nonumber\\
%&\leq n_\tau\cdot \exp\left\{-4\ln(T\varepsilon_\tau^2)\right\}
\end{align}

Combining all regret parts we complete the proof of Lemma \ref{lem:case2}.

\subsection{Proof of Lemma \ref{lem:case1}}
The regret for all outer iterations after $\tau$ (conditioned on $\mathcal E_1(\tau): \theta^*<a_\tau<b_\tau$) consists of two parts: the regret from exploiting $\mathcal L_{y_{\tau'}}$ for $\tau'\geq \tau$,
and the regret from exploring $\mathcal L_{a_{\tau'}}$.

For any $\tau'\in\{\tau,\tau+1,\cdots,\tau_0\}$, the expected regret from exploiting $\mathcal L_{y_\tau'}$ can always be upper bounded by $O(\varepsilon_{\tau'}^{-1}\log(T\varepsilon_{\tau'}^2))$
by the same analysis in the proof of Lemma \ref{lem:case2} (more specifically the array of inequalities leading to Eqs.~(\ref{eq:key2-case2}) and (\ref{eq:regret-2})),
regardless of the values of $\alpha$ and $\beta$.
This corresponds to the $\sum_{\tau'=\tau}^{\tau_0}O(\varepsilon_{\tau'}^{-1}\log(T\varepsilon_{\tau'}^2))$ term in Lemma \ref{lem:case1}.

We next upper bound the expected regret incurred by exploring $\mathcal L_{a_{\tau'}}$ for all $\tau'=\tau,\tau+1,\cdots,\tau_0$.
Because $a_\tau-F(a_\tau)=\beta$ by the definition of $\psi_{\tau}^1(\alpha,\beta)$, the expected regret incurred by exploring $\mathcal L_{a_{\tau'}}$, $\tau'\in\{\tau,\tau+1,\cdots,\tau_0\}$
is at most $\beta T$ \emph{assuming $a_\tau=a_{\tau+1}=\cdots=a_{\tau_0}$}.
It then remains to bound the additional regret incurred by the movements of $a_{\tau'}$ in subsequent outer iterations.

Let $\mathcal W=\{\tau_1',\tau_2',\cdots, \tau_\ell'\}$ be outer iterations at which the update rule $a_{\tau+1}\gets x_\tau$ is applied.
We then have the following observations:
\begin{enumerate}[leftmargin=0.2in]
\item Each $\tau'\in\mathcal W$ would incur an additional regret upper bounded by $\Delta_{\tau'}T$, where $\Delta_{\tau'}=y_{\tau'}-F(y_{\tau'}) \geq \varepsilon_{\tau'}$;
\item For each $\tau'\in\{\tau,\tau+1,\cdots,\tau_0\}$, the probability update $a_{\tau'+1}\gets x_{\tau'}$ is applied is at most $\exp\{-n_{\tau'}\Delta_{\tau'}\}$,
using the same analysis in the proof of Lemma \ref{lem:case2} (more specifically the array of inequalities leading to Eq.~(\ref{eq:right-trisection-prob})).
\end{enumerate}
Summarizing the above observations, by the law of total expectation the expected regret from exploring $\mathcal L_{a_{\tau'}}$ at subsequent iterations $\tau'\geq \tau$ can be upper bounded by
$\beta T + \sum_{\tau'=\tau}^{\tau_0}\sup_{\Delta>\varepsilon_\tau}\Delta T\exp\{-n_\tau\Delta^2\}$.

\subsection{Proof of Lemma \ref{lem:case3}}
Because $a_\tau = \theta^*-\alpha < \theta^*$, by Lemma \ref{lem:Fmonotonic-left} we have $F(a_\tau) \geq a_\tau =\theta^*-\alpha = F(\theta^*)-\alpha$.
Subsequently, $F(S^*)-F(a_\tau) \leq \alpha$ thanks to Lemma \ref{lem:rrequal}.
Also note that conditioned on $\mathcal E_3(\tau)$, the revenue levels explored or exploited at each time epoch $t\in\mathcal T(\tau')$, $\tau\leq \tau'\leq \tau_0$
are sandwiched between $a_\tau$ and $\theta^*$, and therefore $R(S^*)-R(S_t)\leq \alpha$.
Hence, $\psi_\tau^3(\alpha,\beta) \leq \alpha\cdot\mathbb E \sum_{\tau'=\tau}^{\tau_0}|\mathcal T(\tau')| \leq \alpha T$.

\section{Proofs of technical lemmas in Sec.~\ref{sec:lower}}

\subsection{Proof of Lemma \ref{lem:minimax}}
We first state a lemma that upper bounds the KL divergence under $P_0$ and $P_1$ for arbitrary assortment selections $S\in\mathbb S$.
\begin{lemma}
For any $S\in\mathbb S$ let $P_0(S)$ and $P_1(S)$ be the distribution of the purchasing action under $P_0$ and $P_1$, respectively.
Then $\kl(P_0(S)\|P_1(S)) \leq 1/18T$.
\label{lem:KL}
\end{lemma}
\begin{proof}[Proof of Lemma \ref{lem:KL}]

If $1\notin S$ then $P_0(S)\equiv P_1(S)$ and therefore $\kl(P_0(S)\|P_1(S))=0$.
In addition, because $v_i=r_i=0$ for all $i\geq 3$, the items apart from $1$ and $2$ in $S$ do not affect the distribution of the purchasing action under both $P_0$ and $P_1$.
Therefore, it suffices to compute $\kl(P_0(\{1\})\|P_1(\{1\}))$ and $\kl(P_0(\{1,2\}\|P_1(\{1,2\}))$.

Before delving into detailed calculations, we first state a simple proposition bounding the KL divergence between two categorical distributions.
It is simple to verify. %, with a proof given in \cite{chen2017note}.
\begin{proposition}
Let $P$ and $Q$ be two categorical distributions on $J$ items, with parameters $p_1,\cdots,p_J$ and $q_1,\cdots,q_J$ respectively.
Denote also $\varepsilon_j := p_j-q_j$. Then
$\kl(P\|Q)\leq \sum_{j=1}^J \varepsilon_j^2/q_j$.
\label{prop:kl}
\end{proposition}

We first consider $\kl(P_0(\{1\})\|P_1(\{1\}))$.
By definition, $P_0(i=1|\{1\})\leq 1/2-1/24\sqrt{T}$ and $P_1[i=2|\{2\}]\leq 1/2+1/24\sqrt{T}$.
Also, $\min_{i=0,1}\{P_1(i|\{1\})\} \geq 1/3$.
Subsequently,
\begin{align}
\kl(P_0(\{1\})\|P_1(\{1\}))
&\leq 2\times\frac{1/144T}{1/3} \leq \frac{1}{24T} \leq \frac{1}{18T}.
\end{align}

We next consider $\kl(P_0(\{1,2\})\|P_1(\{1,2\}))$.
Note that $P_0(i=0|\{1,2\})>P_1(i=0|\{1,2\})$, $P_0(i=1|\{1,2\})<P_1(i=1|\{1,2\})$ and $P_0(i=2|\{1,2\})>P_1(i=2|\{1,2\})$.
Also, $P_0(i=1|\{1,2\})\leq 1/3-1/48\sqrt{T}$, $P_1(i=1|\{1,2\})\geq 1/3+1/48\sqrt{T}$ and $\min_{0\leq i\leq 2}\{P_1(i|\{1,2\})\}\geq 1/4$. Subsequently,
\begin{align}
\kl(P_0(\{1,2\})\|P_1(\{1,2\}))
&\leq 3\times\frac{1/576T}{1/4} \leq \frac{1}{48T} \leq \frac{1}{18T}.
\end{align}
The lemma is thus proved.
\end{proof}

We are now ready to prove Lemma \ref{lem:minimax}.
\begin{proof}[Proof of Lemma \ref{lem:minimax}]

Denote $\|P-Q\|_\tv := 2\sup_A|P(A)-Q(A)|$ as the total variation norm between $P$ and $Q$,
and let $P_0^{\otimes T},P_1^{\otimes T}$ denote the distribution of $\{i_t|S_t\}_{t=1}^T$ parameterized by $P_0$ and $P_1$.
By Pinsker's inequality and the conditional independence of $i_t$ conditioned on $S_t$, we have
\begin{align}
\|P_0^{\otimes T}-P_1^{\otimes T}\|_{\tv}
&\leq \sqrt{2\kl(P_0^{\otimes T}\|P_1^{\otimes T})}
\leq \sup_{S^{(1)},\cdots,S_t}\sqrt{2\prod_{t=1}^T{\kl(P_0(S_t)\|P_1(S_t))}}\\
&\leq \sqrt{2T}\cdot\sup_{S}\sqrt{\kl(P_0(S)\|P_1(S))} \leq \sqrt{2T}\cdot \sqrt{1/18T} \leq 1/3.
\end{align}
Using Le Cam's inequality we have
\begin{equation}
\inf_{\hat\psi}\max_{j=0,1} P_j\left[\hat\psi\neq j\right] \geq \frac{1}{2}\left(1-\|P_0^{\otimes T}-P_1^{\otimes T}\|_{\tv}\right) \geq \frac{1}{3},
\end{equation}
\end{proof}

\subsection{Proof of Lemma \ref{lem:reduction}}

Denote $\wp_0 := 1/T\cdot \sum_{t=1}^T{\mathbb I[1\in S_t,2\notin S_t]}$,  $\wp_1 := 1/T\cdot \sum_{t=1}^T{\mathbb I[1,2\in S_t]}$,
 $\wp_2 := 1/T\cdot\sum_{t=1}^T{\mathbb I[2\in S_t,1\notin S_t]}$
 and $\bar\wp := 1/T\cdot \sum_{t=1}^T{\mathbb I[1,2\notin S_t]}$.
Because the four events partition the entire probability space, we have $\wp_0+\wp_1+\wp_2+\bar\wp= 1$.
In addition, it is easy to verify that $S^*=\{1\}$ under $P_0$ and under $P_1$.
Subsequently,
\begin{eqnarray*}
\frac{\reg_\pi(T)}{T}& \leq& \frac{\wp_0}{12\sqrt{T}}+\frac{\wp_2+\bar\wp}{24}\;\;\;\;\; \text{under $P_0$};\\
\frac{\reg_\pi(T)}{T} &\leq& \frac{\wp_1}{48\sqrt{T}}+\frac{\wp_2+\bar\wp}{6}\;\;\;\;\; \text{under $P_1$}.
\end{eqnarray*}
Using Markov's inequality and the fact that $\reg_\pi(T)\leq \sqrt{T}/384$ under both $P_0$ and $P_1$, we have
\begin{equation}
P_0\left[\frac{\wp_0}{12\sqrt{T}}+\frac{\wp_2+\bar\wp}{24} > \frac{1}{96\sqrt{T}}\right] \leq \frac{1}{4}\;\;\;\;\text{and}\;\;\;\;
P_1\left[\frac{\wp_1}{48\sqrt{T}}+\frac{\wp_2+\bar\wp}{6} > \frac{1}{96\sqrt{T}}\right] \leq \frac{1}{4}.
\end{equation}
Subsequently, because $\wp_0+\wp_1+\wp_2+\bar\wp=1$, we know that $\wp_0>1/2$ with probability $\geq 2/3$ under $P_0$ and
$\wp_0<1/2$ with probability $\geq 2/3$ under $P_1$.
Define $\hat\psi$ as
\begin{equation}
\hat\psi := \left\{\begin{array}{ll} 0& \text{if $\wp_0\geq 1/2$};\\ 1& \text{if $\wp_0 < 1/2$}.\end{array}\right.
\end{equation}
The estimator $\hat\psi$ then satisfies Lemma \ref{lem:reduction} by the above argument.

\section{Concentration inequalities}

The following lemma is the celebrated Hoeffding's inequality \cite{hoeffding1963probability}.
\begin{lemma}
Suppose $X_1,\cdots,X_n$ are i.i.d.~random variables with mean $\mu$ and satisfy $a\leq X_i\leq b$ almost surely for all $i\in[n]$.
Then for any $t>0$,
\begin{equation}
\Pr\left[\left|\frac{1}{n}\sum_{i=1}^n X_i-\mu\right| > t \right] \leq 2\exp\left\{-\frac{2nt^2}{(b-a)^2}\right\}.
\end{equation}
\label{lem:hoeffding}
\end{lemma}

The following lemma is the Hoeffding's maximal inequality, also by \cite{hoeffding1963probability}.
\begin{lemma}
Let $X_1,\cdots,X_n$ be i.i.d.~random variables with mean $\mu$ and satisfy $a\leq X_i\leq b$ almost surely for all $i\in[n]$. Then for any $t>0$, 
\begin{equation}
\Pr\left[\forall i\in[n], X_1+\cdots+X_i \geq i\cdot\mu + t\right] \leq \exp\left\{-\frac{2t^2}{n(b-a)^2}\right\}.
\end{equation}
\label{lem:hoeffding-maximal}
\end{lemma}

\end{document}